\documentclass{article}
\usepackage[latin9]{inputenc}
\usepackage{float}
\usepackage{units}
\usepackage{amsmath}
\usepackage{amsthm}
\usepackage{amssymb}
\usepackage{graphicx}
\usepackage{esint}
\usepackage[authoryear]{natbib}
\usepackage{xargs}[2008/03/08]
\usepackage[unicode=true,
 bookmarks=false,
 breaklinks=false,pdfborder={0 0 1},backref=section,colorlinks=false]
 {hyperref}

\makeatletter

\providecommand{\tabularnewline}{\\}
\newcommand{\lyxdot}{.}

\floatstyle{ruled}
\newfloat{algorithm}{tbp}{loa}
\providecommand{\algorithmname}{Algorithm}
\floatname{algorithm}{\protect\algorithmname}

\theoremstyle{plain}
\newtheorem{thm}{\protect\theoremname}
  \theoremstyle{plain}
  \newtheorem{prop}[thm]{\protect\propositionname}
 \usepackage{algolyx}

\usepackage{iclr2018_conference}
\usepackage{times}\usepackage{url}

\title{MGAN: Training Generative Adversarial Nets with Multiple Generators}

\author{Quan Hoang \\
University of Massachusetts-Amherst\\
Amherst, MA 01003, USA \\
\texttt{qhoang@umass.edu} \\
\AND
Tu Dinh Nguyen, Trung Le, Dinh Phung \\
PRaDA Centre,  Deakin University\\
Geelong, Australia \\
\texttt{\{tu.nguyen,trung.l,dinh.phung\}@deakin.edu.au} \\
}

\newcommand{\new}{\marginpar{NEW}}

\iclrfinalcopy 

\usepackage{graphicx}

\usepackage{colortbl} 
\definecolor{header_color}{rgb}{0.74,0.88,0.91}
\definecolor{even_color}{rgb}{0.9,0.9,0.9}
\definecolor{subheader_color}{rgb}{0.85,0.93,0.95}
\definecolor{childheader_color}{rgb}{1.0,0.93,0.87}

\definecolor{ccolor_best}{rgb}{1.0,0.9,0.9}
\definecolor{ccolor_wrong}{rgb}{1.0,0.85,0.85}

\usepackage{ifthen}
\renewenvironment{figure}[1][]{%
 \ifthenelse{\equal{#1}{}}{%
   \@float{figure}
 }{%
   \@float{figure}[#1]%
 }%
 \centering
}{%
 \end@float
}

\renewenvironment{table}[1][]{%
 \ifthenelse{\equal{#1}{}}{%
   \@float{table}
 }{%
   \@float{table}[#1]%
 }%
 \centering
}{%
 \end@float
}

\@ifundefined{showcaptionsetup}{}{%
 \PassOptionsToPackage{caption=false}{subfig}}
\usepackage{subfig}
\makeatother

  \providecommand{\propositionname}{Proposition}
\providecommand{\theoremname}{Theorem}

\begin{document}
\maketitle
\begin{abstract}
We propose in this paper a new approach to train the Generative Adversarial
Nets (GANs) with a mixture of generators to overcome the mode collapsing
problem. The main intuition is to employ multiple generators, instead
of using a single one as in the original GAN. The idea is simple,
yet proven to be extremely effective at covering diverse data modes,
easily overcoming the mode collapsing problem and delivering state-of-the-art
results. A minimax formulation was able to establish among a classifier,
a discriminator, and a set of generators in a similar spirit with
GAN. Generators create samples that are intended to come from the
same distribution as the training data, whilst the discriminator determines
whether samples are true data or generated by generators, and the
classifier specifies which generator a sample comes from. The distinguishing
feature is that internal samples are created from multiple generators,
and then one of them will be randomly selected as final output similar
to the mechanism of a probabilistic mixture model. We term our method
\emph{Mixture Generative Adversarial Nets} (MGAN). We develop theoretical
analysis to prove that, at the equilibrium, the Jensen-Shannon divergence
(JSD) between the mixture of generators' distributions and the empirical
data distribution is minimal, whilst the JSD among generators' distributions
is maximal, hence effectively avoiding the mode collapsing problem.
By utilizing parameter sharing, our proposed model adds minimal computational
cost to the standard GAN, and thus can also efficiently scale to large-scale
datasets. We conduct extensive experiments on synthetic 2D data and
natural image databases (CIFAR-10, STL-10 and ImageNet) to demonstrate
the superior performance of our MGAN in achieving state-of-the-art
Inception scores over latest baselines, generating diverse and appealing
recognizable objects at different resolutions, and specializing in
capturing different types of objects by the generators.
\end{abstract}
\newcommand{\sidenote}[1]{\marginpar{\small \emph{\color{Medium}#1}}}

\global\long\def\se{\hat{\text{se}}}

\global\long\def\interior{\text{int}}

\global\long\def\boundary{\text{bd}}

\global\long\def\ML{\textsf{ML}}

\global\long\def\GML{\mathsf{GML}}

\global\long\def\HMM{\mathsf{HMM}}

\global\long\def\support{\text{supp}}

\global\long\def\new{\text{*}}

\global\long\def\stir{\text{Stirl}}

\global\long\def\mA{\mathcal{A}}

\global\long\def\mB{\mathcal{B}}

\global\long\def\expect{\mathbb{E}}

\global\long\def\mF{\mathcal{F}}

\global\long\def\mK{\mathcal{K}}

\global\long\def\mH{\mathcal{H}}

\global\long\def\mX{\mathcal{X}}

\global\long\def\mZ{\mathcal{Z}}

\global\long\def\mS{\mathcal{S}}

\global\long\def\Ical{\mathcal{I}}

\global\long\def\mT{\mathcal{T}}

\global\long\def\Pcal{\mathcal{P}}

\global\long\def\dist{d}

\global\long\def\HX{\entro\left(X\right)}
 \global\long\def\entropyX{\HX}

\global\long\def\HY{\entro\left(Y\right)}
 \global\long\def\entropyY{\HY}

\global\long\def\HXY{\entro\left(X,Y\right)}
 \global\long\def\entropyXY{\HXY}

\global\long\def\mutualXY{\mutual\left(X;Y\right)}
 \global\long\def\mutinfoXY{\mutualXY}

\global\long\def\given{\mid}

\global\long\def\gv{\given}

\global\long\def\goto{\rightarrow}

\global\long\def\asgoto{\stackrel{a.s.}{\longrightarrow}}

\global\long\def\pgoto{\stackrel{p}{\longrightarrow}}

\global\long\def\dgoto{\stackrel{d}{\longrightarrow}}

\global\long\def\lik{\mathcal{L}}

\global\long\def\logll{\mathit{l}}

\global\long\def\bigcdot{\raisebox{-0.5ex}{\scalebox{1.5}{\ensuremath{\cdot}}}}

\global\long\def\sig{\textrm{sig}}

\global\long\def\likelihood{\mathcal{L}}

\global\long\def\vectorize#1{\mathbf{#1}}

\global\long\def\vt#1{\mathbf{#1}}

\global\long\def\gvt#1{\boldsymbol{#1}}

\global\long\def\idp{\ \bot\negthickspace\negthickspace\bot\ }
 \global\long\def\cdp{\idp}

\global\long\def\das{}

\global\long\def\id{\mathbb{I}}

\global\long\def\idarg#1#2{\id\left\{  #1,#2\right\}  }

\global\long\def\iid{\stackrel{\text{iid}}{\sim}}

\global\long\def\bzero{\vt 0}

\global\long\def\bone{\mathbf{1}}

\global\long\def\a{\mathrm{a}}

\global\long\def\ba{\mathbf{a}}

\global\long\def\b{\mathrm{b}}

\global\long\def\bb{\mathbf{b}}

\global\long\def\B{\mathrm{B}}

\global\long\def\boldm{\boldsymbol{m}}

\global\long\def\c{\mathrm{c}}

\global\long\def\C{\mathrm{C}}

\global\long\def\d{\mathrm{d}}

\global\long\def\D{\mathrm{D}}

\global\long\def\N{\mathrm{N}}

\global\long\def\h{\mathrm{h}}

\global\long\def\H{\mathrm{H}}

\global\long\def\bH{\mathbf{H}}

\global\long\def\K{\mathrm{K}}

\global\long\def\M{\mathrm{M}}

\global\long\def\bff{\vt f}

\global\long\def\bx{\mathbf{\mathbf{x}}}

\global\long\def\bl{\boldsymbol{l}}

\global\long\def\s{\mathrm{s}}

\global\long\def\T{\mathrm{T}}

\global\long\def\bu{\mathbf{u}}

\global\long\def\v{\mathrm{v}}

\global\long\def\bv{\mathbf{v}}

\global\long\def\bo{\boldsymbol{o}}

\global\long\def\bh{\mathbf{h}}

\global\long\def\bs{\boldsymbol{s}}

\global\long\def\x{\mathrm{x}}

\global\long\def\bx{\mathbf{x}}

\global\long\def\bz{\mathbf{z}}

\global\long\def\hbz{\hat{\bz}}

\global\long\def\z{\mathrm{z}}

\global\long\def\y{\mathrm{y}}

\global\long\def\bxnew{\boldsymbol{y}}

\global\long\def\bX{\boldsymbol{X}}

\global\long\def\tbx{\tilde{\bx}}

\global\long\def\by{\boldsymbol{y}}

\global\long\def\bY{\boldsymbol{Y}}

\global\long\def\bZ{\boldsymbol{Z}}

\global\long\def\bU{\boldsymbol{U}}

\global\long\def\bn{\boldsymbol{n}}

\global\long\def\bV{\boldsymbol{V}}

\global\long\def\bI{\boldsymbol{I}}

\global\long\def\J{\mathrm{J}}

\global\long\def\bJ{\mathbf{J}}

\global\long\def\w{\mathrm{w}}

\global\long\def\bw{\vt w}

\global\long\def\bW{\mathbf{W}}

\global\long\def\balpha{\gvt{\alpha}}

\global\long\def\bdelta{\boldsymbol{\delta}}

\global\long\def\bsigma{\gvt{\sigma}}

\global\long\def\bbeta{\gvt{\beta}}

\global\long\def\bmu{\gvt{\mu}}

\global\long\def\btheta{\boldsymbol{\theta}}

\global\long\def\blambda{\boldsymbol{\lambda}}

\global\long\def\bgamma{\boldsymbol{\gamma}}

\global\long\def\bpsi{\boldsymbol{\psi}}

\global\long\def\bphi{\boldsymbol{\phi}}

\global\long\def\bpi{\boldsymbol{\pi}}

\global\long\def\bomega{\boldsymbol{\omega}}

\global\long\def\bepsilon{\boldsymbol{\epsilon}}

\global\long\def\btau{\boldsymbol{\tau}}

\global\long\def\bxi{\boldsymbol{\xi}}

\global\long\def\realset{\mathbb{R}}

\global\long\def\realn{\realset^{n}}

\global\long\def\integerset{\mathbb{Z}}

\global\long\def\natset{\integerset}

\global\long\def\integer{\integerset}

\global\long\def\natn{\natset^{n}}

\global\long\def\rational{\mathbb{Q}}

\global\long\def\rationaln{\rational^{n}}

\global\long\def\complexset{\mathbb{C}}

\global\long\def\comp{\complexset}

\global\long\def\compl#1{#1^{\text{c}}}

\global\long\def\and{\cap}

\global\long\def\compn{\comp^{n}}

\global\long\def\comb#1#2{\left({#1\atop #2}\right) }

\global\long\def\nchoosek#1#2{\left({#1\atop #2}\right)}

\global\long\def\param{\vt w}

\global\long\def\Param{\Theta}

\global\long\def\meanparam{\gvt{\mu}}

\global\long\def\Meanparam{\mathcal{M}}

\global\long\def\meanmap{\mathbf{m}}

\global\long\def\logpart{A}

\global\long\def\simplex{\Delta}

\global\long\def\simplexn{\simplex^{n}}

\global\long\def\dirproc{\text{DP}}

\global\long\def\ggproc{\text{GG}}

\global\long\def\DP{\text{DP}}

\global\long\def\ndp{\text{nDP}}

\global\long\def\hdp{\text{HDP}}

\global\long\def\gempdf{\text{GEM}}

\global\long\def\rfs{\text{RFS}}

\global\long\def\bernrfs{\text{BernoulliRFS}}

\global\long\def\poissrfs{\text{PoissonRFS}}

\global\long\def\grad{\gradient}
 \global\long\def\gradient{\nabla}

\global\long\def\partdev#1#2{\partialdev{#1}{#2}}
 \global\long\def\partialdev#1#2{\frac{\partial#1}{\partial#2}}

\global\long\def\partddev#1#2{\partialdevdev{#1}{#2}}
 \global\long\def\partialdevdev#1#2{\frac{\partial^{2}#1}{\partial#2\partial#2^{\top}}}

\global\long\def\closure{\text{cl}}

\global\long\def\cpr#1#2{\Pr\left(#1\ |\ #2\right)}

\global\long\def\var{\text{Var}}

\global\long\def\Var#1{\text{Var}\left[#1\right]}

\global\long\def\cov{\text{Cov}}

\global\long\def\Cov#1{\cov\left[ #1 \right]}

\global\long\def\COV#1#2{\underset{#2}{\cov}\left[ #1 \right]}

\global\long\def\corr{\text{Corr}}

\global\long\def\sst{\text{T}}

\global\long\def\SST{\sst}

\global\long\def\ess{\mathbb{E}}

\global\long\def\Ess#1{\ess\left[#1\right]}

\newcommandx\ESS[2][usedefault, addprefix=\global, 1=]{\underset{#2}{\ess}\left[#1\right]}

\global\long\def\fisher{\mathcal{F}}

\global\long\def\bfield{\mathcal{B}}
 \global\long\def\borel{\mathcal{B}}

\global\long\def\bernpdf{\text{Bernoulli}}

\global\long\def\betapdf{\text{Beta}}

\global\long\def\dirpdf{\text{Dir}}

\global\long\def\gammapdf{\text{Gamma}}

\global\long\def\gaussden#1#2{\text{Normal}\left(#1, #2 \right) }

\global\long\def\gauss{\mathbf{N}}

\global\long\def\gausspdf#1#2#3{\text{Normal}\left( #1 \lcabra{#2, #3}\right) }

\global\long\def\multpdf{\text{Mult}}

\global\long\def\poiss{\text{Pois}}

\global\long\def\poissonpdf{\text{Poisson}}

\global\long\def\pgpdf{\text{PG}}

\global\long\def\wshpdf{\text{Wish}}

\global\long\def\iwshpdf{\text{InvWish}}

\global\long\def\nwpdf{\text{NW}}

\global\long\def\niwpdf{\text{NIW}}

\global\long\def\studentpdf{\text{Student}}

\global\long\def\unipdf{\text{Uni}}

\global\long\def\transp#1{\transpose{#1}}
 \global\long\def\transpose#1{#1^{\mathsf{T}}}

\global\long\def\mgt{\succ}

\global\long\def\mge{\succeq}

\global\long\def\idenmat{\mathbf{I}}

\global\long\def\trace{\mathrm{tr}}

\global\long\def\argmax#1{\underset{_{#1}}{\text{argmax}} }

\global\long\def\argmin#1{\underset{_{#1}}{\text{argmin}\ } }

\global\long\def\diag{\text{diag}}

\global\long\def\norm{}

\global\long\def\spn{\text{span}}

\global\long\def\vtspace{\mathcal{V}}

\global\long\def\field{\mathcal{F}}
 \global\long\def\ffield{\mathcal{F}}

\global\long\def\inner#1#2{\left\langle #1,#2\right\rangle }
 \global\long\def\iprod#1#2{\inner{#1}{#2}}

\global\long\def\dprod#1#2{#1 \cdot#2}

\global\long\def\norm#1{\left\Vert #1\right\Vert }

\global\long\def\entro{\mathbb{H}}

\global\long\def\entropy{\mathbb{H}}

\global\long\def\Entro#1{\entro\left[#1\right]}

\global\long\def\Entropy#1{\Entro{#1}}

\global\long\def\mutinfo{\mathbb{I}}

\global\long\def\relH{\mathit{D}}

\global\long\def\reldiv#1#2{\relH\left(#1||#2\right)}

\global\long\def\KL{KL}

\global\long\def\KLdiv#1#2{\KL\left(#1\parallel#2\right)}
 \global\long\def\KLdivergence#1#2{\KL\left(#1\ \parallel\ #2\right)}

\global\long\def\crossH{\mathcal{C}}
 \global\long\def\crossentropy{\mathcal{C}}

\global\long\def\crossHxy#1#2{\crossentropy\left(#1\parallel#2\right)}

\global\long\def\breg{\text{BD}}

\global\long\def\lcabra#1{\left|#1\right.}

\global\long\def\lbra#1{\lcabra{#1}}

\global\long\def\rcabra#1{\left.#1\right|}

\global\long\def\rbra#1{\rcabra{#1}}

\global\long\def\model{\text{MGAN}}

\section{Introduction}

Generative Adversarial Nets (GANs) \citep{goodfellow2014generative}
are a recent novel class of deep generative models that are successfully
applied to a large variety of applications such as image, video generation,
image inpainting, semantic segmentation, image-to-image translation,
and text-to-image synthesis, to name a few \citep{goodfellow2016nips}.
From the game theory metaphor, the model consists of a discriminator
and a generator playing a two-player minimax game, wherein the generator
aims to generate samples that resemble those in the training data
whilst the discriminator tries to distinguish between the two as narrated
in \citep{goodfellow2014generative}. Training GAN, however, is challenging
as it can be easily trapped into the mode collapsing problem where
the generator only concentrates on producing samples lying on a few
modes instead of the whole data space \citep{goodfellow2016nips}.

Many GAN variants have been recently proposed to address this problem.
They can be grouped into two main categories: training either a single
generator or many generators. Methods in the former include modifying
the discriminator's objective \citep{salimans2016improved,metz2016unrolled},
modifying the generator's objective \citep{warde2016improving}, or
employing additional discriminators to yield more useful gradient
signals for the generators \citep{tu_etal_nips17_d2gan,durugkar2016generative}.
The common theme in these variants is that generators are shown, at
equilibrium, to be able to recover the data distribution, but convergence
remains elusive in practice. Most experiments are conducted on toy
datasets or on narrow-domain datasets such as LSUN \citep{yu2015lsun}
or CelebA \citep{liu2015deep}. To our knowledge, only \citet{warde2016improving}
and \citet{tu_etal_nips17_d2gan} perform quantitative evaluation
of models trained on much more diverse datasets such as STL-10 \citep{coates2011analysis}
and ImageNet \citep{russakovsky2015imagenet}.

Given current limitations in the training of single-generator GANs,
some very recent attempts have been made following the multi-generator
approach. \citet{tolstikhin2017adagan} apply boosting techniques
to train a mixture of generators by sequentially training and adding
new generators to the mixture. However, sequentially training many
generators is computational expensive. Moreover, this approach is
built on the implicit assumption that a single-generator GAN can generate
very good images of some modes, so reweighing the training data and
incrementally training new generators will result in a mixture that
covers the whole data space. This assumption is not true in practice
since current single-generator GANs trained on diverse datasets such
as ImageNet tend to generate images of unrecognizable objects. \citet{arora2017generalization}
train a mixture of generators and discriminators, and optimize the
minimax game with the reward function being the weighted average reward
function between any pair of generator and discriminator. This model
is computationally expensive and lacks a mechanism to enforce the
divergence among generators. \citet{ghosh2017multi}  train many
generators by using a multi-class discriminator that, in addition
to detecting whether a data sample is fake, predicts which generator
produces the sample. The objective function in this model punishes
generators for generating samples that are detected as fake but does
not directly encourage generators to specialize in generating different
types of data.

We propose in this paper a novel approach to train a mixture of generators.
Unlike aforementioned multi-generator GANs, our proposed model simultaneously
trains a set of generators with the objective that the mixture of
their induced distributions would approximate the data distribution,
whilst encouraging them to specialize in different data modes. The
result is a novel adversarial architecture formulated as a minimax
game among three parties: a classifier, a discriminator, and a set
of generators. Generators create samples that are intended to come
from the same distribution as the training data, whilst the discriminator
determines whether samples are true data or generated by generators,
and the classifier specifies which generator a sample comes from.
We term our proposed model as \emph{Mixture Generative Adversarial
Nets (MGAN)}. We provide analysis that our model is optimized towards
minimizing the Jensen-Shannon Divergence (JSD) between the mixture
of distributions induced by the generators and the data distribution
while maximizing the JSD among generators.

Empirically, our proposed model can be trained efficiently by utilizing
parameter sharing among generators, and between the classifier and
the discriminator. In addition, simultaneously training many generators
while enforcing JSD among generators helps each of them focus on some
modes of the data space and learn better. Trained on CIFAR-10, each
generator learned to specialize in generating samples from a different
class such as horse, car, ship, dog, bird or airplane. Overall, the
models trained on the CIFAR-10, STL-10 and ImageNet datasets successfully
generated diverse, recognizable objects and achieved state-of-the-art
Inception scores \citep{salimans2016improved}. The model trained
on the CIFAR-10 even outperformed GANs trained in a semi-supervised
fashion \citep{salimans2016improved,odena2016conditional}.

In short, our main contributions are: (i) a novel adversarial model
to efficiently train a mixture of generators while enforcing the JSD
among the generators; (ii) a theoretical analysis that our objective
function is optimized towards minimizing the JSD between the mixture
of all generators' distributions and the real data distribution, while
maximizing the JSD among generators; and (iii) a comprehensive evaluation
on the performance of our method on both synthetic and real-world
large-scale datasets of diverse natural scenes.

\section{Generative Adversarial Nets}

Given the discriminator $D$ and generator $G$, both parameterized
via neural networks, training GAN can be formulated as the following
minimax objective function:
\begin{equation}
\min_{G}\max_{D}\mathbb{E}_{\bx\sim P_{data}\left(\mathbf{x}\right)}\left[\log D\left(\mathbf{x}\right)\right]+\mathbb{E}_{\mathbf{z}\sim P_{\bz}}\left[\log\left(1-D\left(G\left(\mathbf{z}\right)\right)\right)\right]\label{eq:gan}
\end{equation}
where $\bx$ is drawn from data distribution $P_{data}$, $\bz$ is
drawn from a prior distribution $P_{\bz}$. The mapping $G\left(\bz\right)$
induces a generator distribution $P_{model}$ in data space. GAN alternatively
optimizes $D$ and $G$ using stochastic gradient-based learning.
As a result, the optimization order in \ref{eq:gan} can be reversed,
causing the minimax formulation to become maximin. $G$ is therefore
incentivized to map every $\bz$ to a single $\bx$ that is most likely
to be classified as true data, leading to mode collapsing problem.
Another commonly asserted cause of generating less diverse samples
in GAN is that, at the optimal point of $D$, minimizing $G$ is equivalent
to minimizing the JSD between the data and model distributions, which
has been empirically proven to prefer to generate samples around only
a few modes whilst ignoring other modes \citep{huszar2015not,Theis2015}.

\section{Proposed Mixture GANs}

We now present our main contribution of a novel approach that can
effectively tackle mode collapse in GAN. Our idea is to use a mixture
of many distributions rather than a single one as in the standard
GAN, to approximate the data distribution, and simultaneously we enlarge
the divergence of those distributions so that they cover different
data modes.

To this end, an analogy to a game among $\K$ generators $G_{1:\K}$,
a discriminator $D$ and a classifier $C$ can be formulated. Each
generator $G_{k}$ maps $\bz$ to $\bx=G_{k}\left(\bz\right)$, thus
inducing a single distribution $P_{G_{k}}$; and $\K$ generators
altogether induce a mixture over $\K$ distributions, namely $P_{model}$
in the data space. An index $u$ is drawn from a multinomial distribution
$\text{Mult}\left(\bpi\right)$ where $\bpi=\left[\pi_{1},\pi_{2},...,\pi_{\K}\right]$
is the coefficients of the mixture; and then the sample $G_{u}\left(\bz\right)$
is used as the output. Here, we use a predefined $\bpi$ and fix it
instead of learning. The discriminator $D$ aims to distinguish between
this sample and the training samples. The classifier $C$ performs
multi-class classification to classify samples labeled by the indices
of their corresponding generators. We term this whole process and
our model the \emph{Mixture Generative Adversarial Nets} (MGAN).

\begin{figure}
\noindent \centering{}\includegraphics[width=1\textwidth]{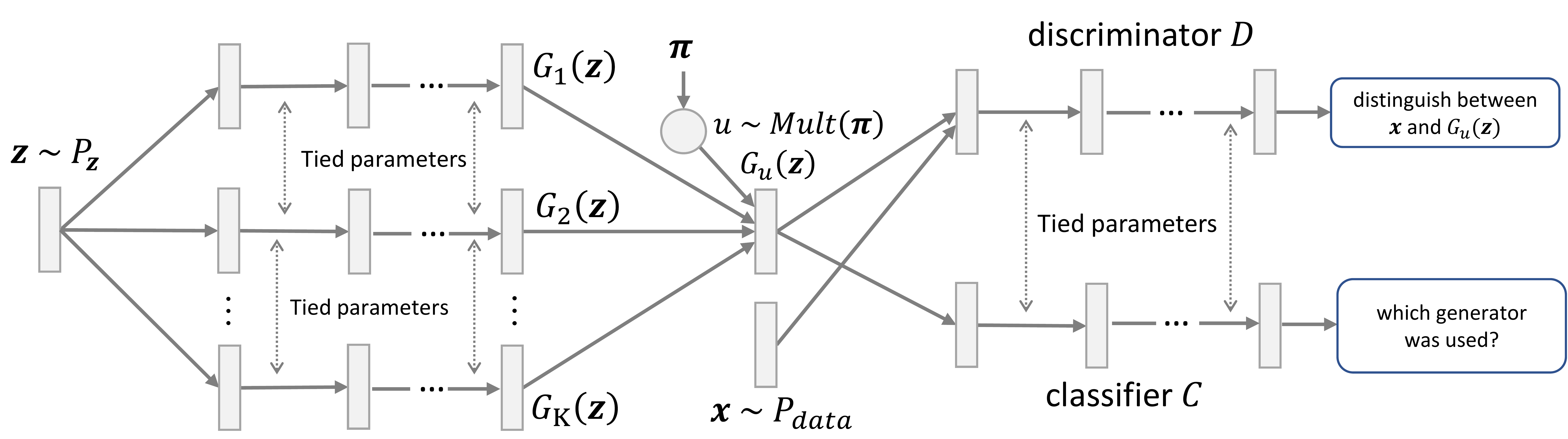}\caption{$\protect\model$'s architecture with $\protect\K$ generators, a
binary discriminator, a multi-class classifier.\label{diagram}}
\end{figure}

Fig.~\ref{diagram} illustrates the general architecture of our proposed
MGAN, where all components are parameterized by neural networks. $G_{k}$
(s) tie their parameters together except the input layer, whilst $C$
and $D$ share parameters except the output layer. This parameter
sharing scheme enables the networks to leverage their common information
such as features at low-level layers that are close to the data layer,
hence helps to train model effectively. In addition, it also minimizes
the number of parameters and adds minimal complexity to the standard
GAN, thus the whole process is still very efficient.

More formally, $D$, $C$ and $G_{1:\K}$ now play the following multi-player
minimax optimization game:
\begin{align}
\min_{G_{1:\K},C}\max_{D}\mathcal{J}\left(G_{1:\K},C,D\right) & =\mathbb{E}_{\bx\sim P_{data}}\left[\log D\left(\mathbf{x}\right)\right]+\mathbb{E}_{\bx\sim P_{model}}\left[\log\left(1-D\left(\mathbf{x}\right)\right)\right]\nonumber \\
 & \,\,\,\,\,\,-\beta\left\{ \sum_{k=1}^{\K}\pi_{k}\mathbb{E}_{\mathbf{x}\sim P_{G_{k}}}\left[\log C_{k}\left(\mathbf{x}\right)\right]\right\} \label{game_formula}
\end{align}
where $C_{k}\left(\mathbf{x}\right)$ is the probability that $\mathbf{x}$
is generated by $G_{k}$ and $\beta>0$ is the diversity hyper-parameter.
The first two terms show the interaction between generators and the
discriminator as in the standard GAN. The last term should be recognized
as a standard softmax loss for a multi-classification setting, which
aims to maximize the entropy for the classifier. This represents the
interaction between generators and the classifier, which encourages
each generator to produce data separable from those produced by other
generators. The strength of this interaction is controlled by $\beta$.
Similar to GAN, our proposed network can be trained by alternatively
updating $D$, $C$ and $G_{1:\K}$. We refer to Appendix~\ref{sec:appendix_framework}
for the pseudo-code and algorithms for parameter learning for our
proposed MGAN.

\subsection{Theoretical Analysis\label{theoretical_analysis}}

Assuming all $C$, $D$ and $G_{1:\K}$ have enough capacity, we show
below that at the equilibrium point of the minimax problem in Eq.~(\ref{game_formula}),
the JSD between the mixture induced by $G_{1:\K}$ and the data distribution
is minimal, i.e. $p_{data}=p_{model}$, and the JSD among $\K$ generators
is maximal, i.e. two arbitrary generators almost never produce the
same data. In what follows we present our mathematical statement and
the sketch of their proofs. We refer to Appendix~\ref{sec:appendix_proofs}
for full derivations.
\begin{prop}
\label{Proposition1}For fixed generators $G_{1}$, $G_{2}$, ...,
$G_{\K}$ and their mixture weights $\pi_{1},\pi_{2},...,\pi_{\K}$,
the optimal solution $C^{*}=C_{1:\K}^{*}$ and $D^{*}$ for $\mathcal{J}\left(G_{1:\K},C,D\right)$
in Eq.~(\ref{game_formula}) are:
\[
C_{k}^{*}\left(\mathbf{x}\right)=\frac{\pi_{k}p_{G_{k}}\left(\mathbf{x}\right)}{\sum_{j=1}^{\K}\pi_{j}p_{G_{j}}\left(\mathbf{x}\right)}\,\,\,\,\textnormal{and}\,\,\,\,D^{*}\left(\mathbf{x}\right)=\frac{p_{data}\left(\mathbf{x}\right)}{p_{data}\left(\mathbf{x}\right)+p_{model}\left(\mathbf{x}\right)}
\]
\end{prop}
\begin{proof}
It can be seen that the solution $C_{k}^{*}$ is a general case of
$D^{*}$ when $D$ classifies samples from two distributions with
equal weight of $\nicefrac{1}{2}$. We refer the proofs for $D^{*}$
to Prop.~1 in \citep{goodfellow2014generative}, and our proof for
$C_{k}^{*}$ to Appendix~\ref{sec:appendix_proofs} in this manuscript.
\end{proof}
Based on Prop.~\ref{Proposition1}, we further show that at the equilibrium
point of the minimax problem in Eq.~(\ref{game_formula}), the optimal
generator $G^{*}=\left[G_{1}^{*},...,G_{\K}^{*}\right]$ induces the
generated distribution $p_{model}^{*}\left(\bx\right)=\sum_{k=1}^{\K}\pi_{k}p_{G_{k}^{*}}\left(\bx\right)$
which is as closest as possible to the true data distribution $p_{data}\left(\bx\right)$
while maintaining the mixture components $p_{G_{k}^{*}}\left(\bx\right)$(s)
as furthest as possible to avoid the mode collapse.
\begin{thm}
\label{thm:equilibrium_point}At the equilibrium point of the minimax
problem in Eq.~(\ref{game_formula}), the optimal $G^{*},D^{*}$,
and $C^{*}$ satisfy
\begin{align}
G^{*} & =\argmin G\left(2\cdot\textnormal{JSD}\left(P_{data}\Vert P_{model}\right)-\beta\cdot\textnormal{JSD}_{\bpi}\left(P_{G_{1}},P_{G_{2}},...,P_{G_{\K}}\right)\right)\label{eq:optimal}\\
 & C_{k}^{*}\left(\mathbf{x}\right)=\frac{\pi_{k}p_{G_{k}^{*}}\left(\mathbf{x}\right)}{\sum_{j=1}^{\K}\pi_{j}p_{G_{j}^{*}}\left(\mathbf{x}\right)}\,\,\,\,\textnormal{and}\,\,\,\,D^{*}\left(\mathbf{x}\right)=\frac{p_{data}\left(\mathbf{x}\right)}{p_{data}\left(\mathbf{x}\right)+p_{model}\left(\mathbf{x}\right)}\nonumber 
\end{align}
\end{thm}
\begin{proof}
Substituting $C_{1:\K}^{*}$ and $D^{*}$ into Eq.~(\ref{game_formula}),
we reformulate the objective function for $G_{1:\K}$ as follows:{\small{}
\begin{alignat}{1}
\mathcal{\mathcal{L}}\left(G_{1:\K}\right) & =\mathbb{E}_{\bx\sim P_{data}}\left[\log\frac{p_{data}\left(\mathbf{x}\right)}{p_{data}\left(\mathbf{x}\right)+p_{model}\left(\mathbf{x}\right)}\right]+\mathbb{E}_{\bx\sim P_{model}}\left[\log\frac{p_{model}\left(\mathbf{x}\right)}{p_{data}\left(\mathbf{x}\right)+p_{model}\left(\mathbf{x}\right)}\right]\nonumber \\
 & \,\,\,\,\,\,-\beta\left\{ \sum_{k=1}^{\K}\pi_{k}\mathbb{E}_{\bx\sim P_{G_{k}}}\left[\log\frac{\pi_{k}p_{G_{k}}\left(\mathbf{x}\right)}{\sum_{j=1}^{\K}\pi_{j}p_{G_{j}}\left(\mathbf{x}\right)}\right]\right\} \nonumber \\
= & 2\cdot\textnormal{JSD}\left(P_{data}\Vert P_{model}\right)-\log4-\beta\left\{ \sum_{k=1}^{\K}\pi_{k}\mathbb{E}_{\bx\sim P_{G_{k}}}\left[\log\frac{p_{G_{k}}\left(\mathbf{x}\right)}{\sum_{j=1}^{\K}\pi_{j}p_{G_{j}}\left(\mathbf{x}\right)}\right]\right\} -\beta\sum_{k=1}^{\K}\pi_{k}\log\pi_{k}\nonumber \\
= & 2\cdot\textnormal{JSD}\left(P_{data}\Vert P_{model}\right)-\beta\cdot\textnormal{JSD}_{\bpi}\left(P_{G_{1}},P_{G_{2}},...,P_{G_{\K}}\right)-\log4-\beta\sum_{k=1}^{\K}\pi_{k}\log\pi_{k}\label{eq:reloss_G}
\end{alignat}
}Since the last two terms in Eq.~(\ref{eq:reloss_G}) are constant,
that concludes our proof.
\end{proof}
This theorem shows that progressing towards the equilibrium is equivalently
to minimizing $\textnormal{JSD}\left(P_{data}\Vert P_{model}\right)$
while maximizing $\textnormal{JSD}_{\bpi}\left(P_{G_{1}},P_{G_{2}},...,P_{G_{\K}}\right)$.
In the next theorem, we further clarify the equilibrium point for
the specific case wherein the data distribution has the form $p_{data}\left(\bx\right)=\sum_{k=1}^{\K}\pi_{k}q_{k}\left(\bx\right)$
where the mixture components $q_{k}\left(\bx\right)$(s) are well-separated
in the sense that $\mathbb{E}_{\bx\sim Q_{k}}\left[q_{j}\left(\bx\right)\right]=0$
for $j\neq k$, i.e., for almost everywhere $\bx$, if $q_{k}\left(\bx\right)>0$
then $q_{j}\left(\bx\right)=0,\,\forall j\neq k$.
\begin{thm}
\label{theorem_global_minimum}If the data distribution has the form:
$p_{data}\left(\bx\right)=\sum_{k=1}^{\K}\pi_{k}q_{k}\left(\bx\right)$
where the mixture components $q_{k}\left(\bx\right)$(s) are well-separated,
the minimax problem in Eq.~(\ref{game_formula}) or the optimization
problem in Eq. (\ref{eq:optimal}) has the following solution:
\[
p_{G_{k}^{*}}\left(\bx\right)=q_{k}\left(\bx\right),\,\forall k=1,\dots,\K\,\,\text{and}\,\,p_{model}\left(\bx\right)=\sum_{k=1}^{\K}\pi_{k}q_{k}\left(\bx\right)=p_{data}\left(\bx\right)
\]
, and the corresponding objective value of the optimization problem
in Eq. (\ref{eq:optimal}) is $-\beta\mathbb{H}\left(\bpi\right)=-\beta\sum_{k=1}^{\K}\pi_{k}\log\frac{1}{\pi_{k}}$,
where $\mathbb{H}\left(\bpi\right)$ is the Shannon entropy.\end{thm}
\begin{proof}
Please refer to our proof in Appendix~\ref{sec:appendix_proofs}
of this manuscript.
\end{proof}
Thm.~\ref{theorem_global_minimum} explicitly offers the optimal
solution for the specific case wherein the real data are generated
from a mixture distribution whose components are well-separated. This
further reveals that if the mixture components are well-separated,
by setting the number of generators as the number of mixtures in data
and maximizing the divergence between the generated components $p_{G_{k}}\left(\bx\right)$(s),
we can exactly recover the mixture components $q_{k}\left(\bx\right)$(s)
using the generated components $p_{G_{k}}\left(\bx\right)$(s), hence
strongly supporting our motivation when developing $\model$. In practice,
$C$, $D$, and $G_{1:\K}$ are parameterized by neural networks and
are optimized in the parameter space rather than in the function space.
As all generators $G_{1:\K}$ share the same objective function, we
can efficiently update their weights using the same backpropagation
passes. Empirically, we set the parameter $\pi_{k}=\frac{1}{\K},\forall k\in\left\{ 1,...,\K\right\} $,
which further minimizes the objective value $-\beta\mathbb{H}\left(\bpi\right)=-\beta\sum_{k=1}^{\K}\pi_{k}\log\frac{1}{\pi_{k}}$
w.r.t $\bpi$ in Thm.~\ref{theorem_global_minimum}. To simplify
the computational graph, we assume that each generator is sampled
the same number of times in each minibatch. In addition, we adopt
the non-saturating heuristic proposed in \citep{goodfellow2014generative}
to train $G_{1:\K}$ by maximizing $\log D\left(G_{k}\left(\mathbf{z}\right)\right)$
instead of minimizing $\log D\left(1-G_{k}\left(\mathbf{z}\right)\right)$
.

\section{Related Work}

Recent attempts to address the mode collapse by modifying the discriminator
include minibatch discrimination \citep{salimans2016improved}, Unrolled
GAN \citep{metz2016unrolled} and Denoising Feature Matching (DFM)
\citep{warde2016improving}. The idea of minibatch discrimination
is to allow the discriminator to detect samples that are noticeably
similar to other generated samples. Although this method can generate
visually appealing samples, it is computationally expensive, thus
normally used in the last hidden layer of discriminator. Unrolled
GAN improves the learning by unrolling computational graph to include
additional optimization steps of the discriminator. It could effectively
reduce the mode collapsing problem, but the unrolling step is expensive,
rendering it unscalable up to large-scale datasets. DFM augments the
objective function of generator with one of a Denoising AutoEncoder
(DAE) that minimizes the reconstruction error of activations at the
penultimate layer of the discriminator. The idea is that gradient
signals from DAE can guide the generator towards producing samples
whose activations are close to the manifold of real data activations.
DFM is surprisingly effective at avoiding mode collapse, but the involvement
of a deep DAE adds considerable computational cost to the model.

An alternative approach is to train additional discriminators. D2GAN
\citep{tu_etal_nips17_d2gan} employs two discriminators to minimize
both Kullback-Leibler (KL) and reverse KL divergences, thus placing
a fair distribution across the data modes. This method can avoid the
mode collapsing problem to a certain extent, but still could not outperform
DFM. Another work uses many discriminators to boost the learning of
generator \citep{durugkar2016generative}. The authors state that
this method is robust to mode collapse, but did not provide experimental
results to support that claim.

Another direction is to train multiple generators. The so-called MIX+GAN
\citep{arora2017generalization} is related to our model in the use
of mixture but the idea is very different. Based on min-max theorem
\citep{neumann1928theorie}, the MIX+GAN trains a mixture of multiple
generators and discriminators with \emph{different parameters} to
play mixed strategies in a min-max game. The total reward of this
game is computed by weighted averaging rewards over all pairs of generator
and discriminator. The lack of parameter sharing renders this method
computationally expensive to train. Moreover, there is no mechanism
to enforce the divergence among generators as in ours.

Some attempts have been made to train a mixture of GANs in a similar
spirit with boosting algorithms. \citet{wang2016ensembles} propose
an additive procedure to incrementally train new GANs on a subset
of the training data that are badly modeled by previous generators.
As the discriminator is expected to classify samples from this subset
as real with high confidence, i.e. $D\left(\mathbf{x}\right)$ is
high, the subset can be chosen to include $\mathbf{x}$ where $D\left(\mathbf{x}\right)$
is larger than a predefined threshold. \citet{tolstikhin2017adagan},
however, show that this heuristic fails to address the mode collapsing
problem. Thus they propose AdaGAN to introduce a robust reweighing
scheme to prepare training data for the next GAN. AdaGAN and boosting-inspired
GANs in general are based on the assumption that a single-generator
GAN can learn to generate impressive images of some modes such as
dogs or cats but fails to cover other modes such as giraffe. Therefore,
removing images of dogs or cats from the training data and train a
next GAN can create a better mixture. This assumption is not true
in practice as current single-generator GANs trained on diverse data
sets such as ImageNet \citep{russakovsky2015imagenet} tend to generate
images of unrecognizable objects.

The most closely related to ours is MAD-GAN \citep{ghosh2017multi}
which trains many generators and uses a multi-class classifier as
the discriminator. In this work, two strategies are proposed to address
the mode collapse: (i) augmenting generator's objective function
with a user-defined similarity based function to encourage different
generators to generate diverse samples, and (ii) modifying discriminator's
objective functions to push different generators towards different
identifiable modes by separating samples of each generator. Our approach
is different in that, rather than modifying the discriminator, we
use an additional classifier that discriminates samples produced by
each generator from those by others under multi-class classification
setting. This nicely results in an optimization problem that maximizes
the JSD among generators, thus naturally enforcing them to generate
diverse samples and effectively avoiding mode collapse.

\section{Experiments}

In this section, we conduct experiments on both synthetic data and
real-world large-scale datasets. The aim of using synthetic data is
to visualize, examine and evaluate the learning behaviors of our proposed
$\model$, whilst using real-world datasets to quantitatively demonstrate
its efficacy and scalability of addressing the mode collapse in a
much larger and wider data space. For fair comparison, we use experimental
settings that are identical to previous work, and hence we quote the
results from the latest state-of-the-art GAN-based models to compare
with ours.

We use TensorFlow \citep{abadi2016tensorflow} to implement our model
and will release the code after publication. For all experiments,
we use: (i) shared parameters among generators in all layers except
for the weights from the input to the first hidden layer; (ii) shared
parameters between discriminator and classifier in all layers except
for the weights from the penultimate layer to the output; (iii) Adam
optimizer \citep{kingma2014adam} with learning rate of 0.0002 and
the first-order momentum of 0.5; (iv) minibatch size of 64 samples
for training discriminators; (v) ReLU activations \citep{nair2010rectified}
for generators; (vi) Leaky ReLU \citep{maas2013rectifier} with slope
of 0.2 for discriminator and classifier; and (vii) weights randomly
initialized from Gaussian distribution $\mathcal{N}(0,0.02\bI)$ and
zero biases. We refer to Appendix~\ref{sec:appendix_exps} for detailed
model architectures and additional experimental results.

\subsection{Synthetic data\label{sec:Synthetic-data}}

In the first experiment, following \citep{tu_etal_nips17_d2gan} we
reuse the experimental design proposed in \citep{metz2016unrolled}
to investigate how well our $\model$ can explore and capture multiple
data modes. The training data is sampled from a 2D mixture of 8 isotropic
Gaussian distributions with a covariance matrix of 0.02$\bI$ and
means arranged in a circle of zero centroid and radius of 2.0. Our
purpose of using such small variance is to create low density regions
and separate the modes.

We employ 8 generators, each with a simple architecture of an input
layer with 256 noise units drawn from isotropic multivariate Gaussian
distribution $\mathcal{N}\left(0,\bI\right)$, and two fully connected
hidden layers with 128 ReLU units each. For the discriminator and
classifier, one hidden layer with 128 ReLU units is used. The diversity
hyperparameter $\beta$ is set to 0.125.

Fig.~\ref{fig:2D_distributions} shows the evolution of 512 samples
generated by our model and baselines through time. It can be seen
that the regular GAN generates data collapsing into a \emph{single}
mode hovering around the valid modes of data distribution, thus reflecting
the mode collapse in GAN as expected. At the same time, UnrolledGAN
\citep{metz2016unrolled}, D2GAN \citep{tu_etal_nips17_d2gan} and
our $\model$ distribute data around \emph{all} 8 mixture components,
and hence demonstrating the abilities to successfully learn multimodal
data in this case. Our proposed model, however, converges much faster
than the other two since it successfully explores and neatly covers
all modes at the early step 15K, whilst two baselines produce samples
cycling around till the last steps. At the end, our $\model$ captures
data modes more precisely than UnrolledGAN and D2GAN since, in each
mode, the UnrolledGAN generates data that concentrate only on several
points around the mode\textquoteright s centroid, thus seems to produce
fewer samples than ours whose samples fairly spread out the entire
mode, but not exceed the boundary whilst the D2GAN still generates
many points scattered between two adjacent modes.

Next we further quantitatively compare the quality of generated data.
Since we know the true distribution $P_{data}$ in this case, we employ
two measures, namely symmetric Kullback-Leibler (KL) divergence and
Wasserstein distance. These measures compute the distance between
the normalized histograms of 10,000 points generated from the model
to true $P_{data}$. Figs.~\ref{fig:2D_symmetric_KL} and \ref{fig:2D_Wasserstein}
again clearly demonstrate the superiority of our approach over GAN,
UnrolledGAN and D2GAN w.r.t both distances (lower is better); notably
the Wasserstein distances from ours and D2GAN's to the true distribution
almost reduce to zero, and at the same time, our symmetric KL metric
is significantly better than that of D2GAN. These figures also show
the stability of our $\model$ (black curves) and D2GAN (red curves)
during training as they are much less fluctuating compared with GAN
(green curves) and UnrolledGAN (blue curves).

\begin{figure}
\noindent \begin{centering}
\begin{minipage}[t]{0.3\textwidth}%
\noindent \begin{center}
\subfloat[Symmetric KL divergence.\label{fig:2D_symmetric_KL}]{\noindent \begin{centering}
\includegraphics[width=1\textwidth]{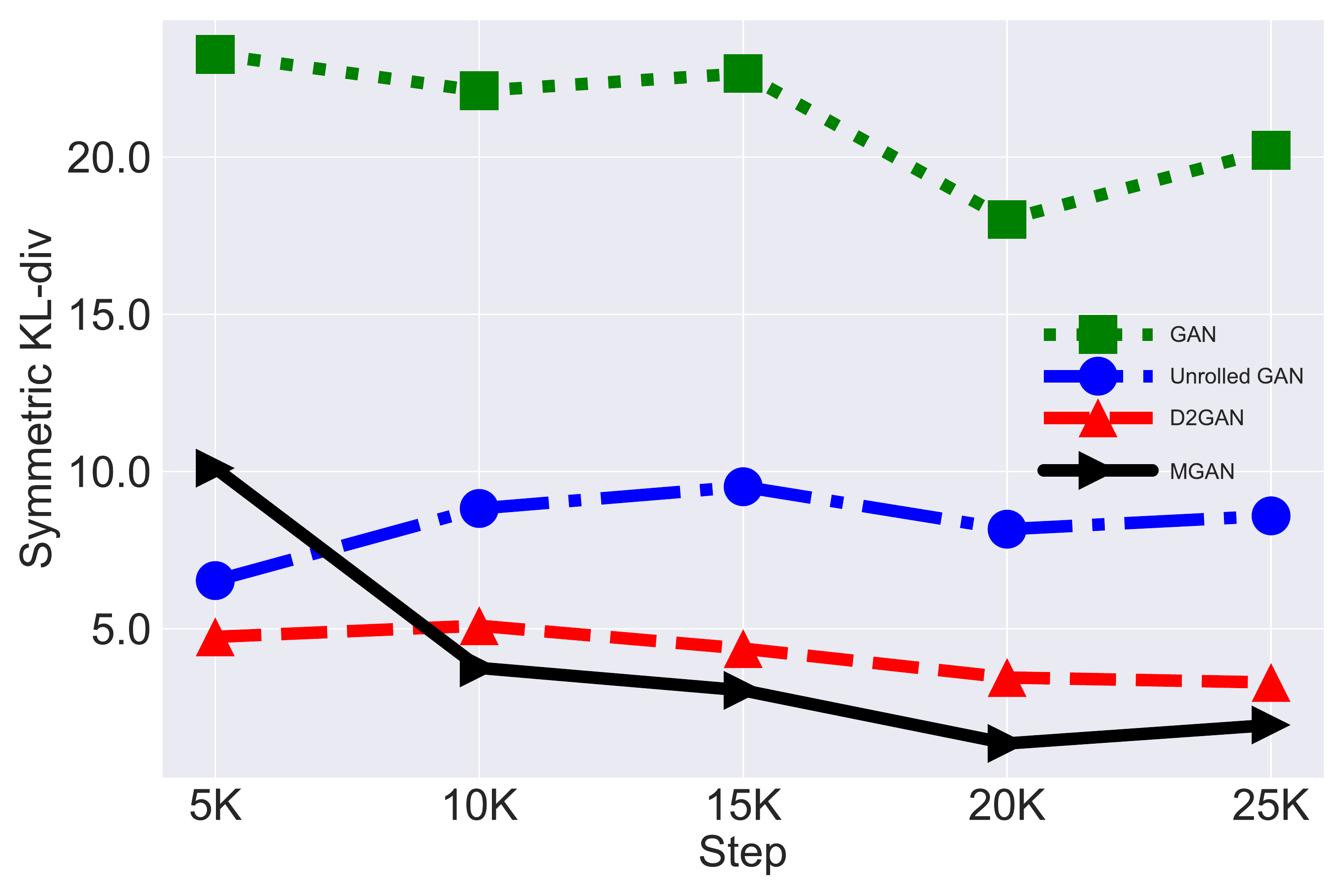}
\par\end{centering}

\noindent \centering{}}
\par\end{center}

\noindent \begin{center}
\subfloat[Wasserstein distance.\label{fig:2D_Wasserstein}]{\noindent \begin{centering}
\includegraphics[width=1\textwidth]{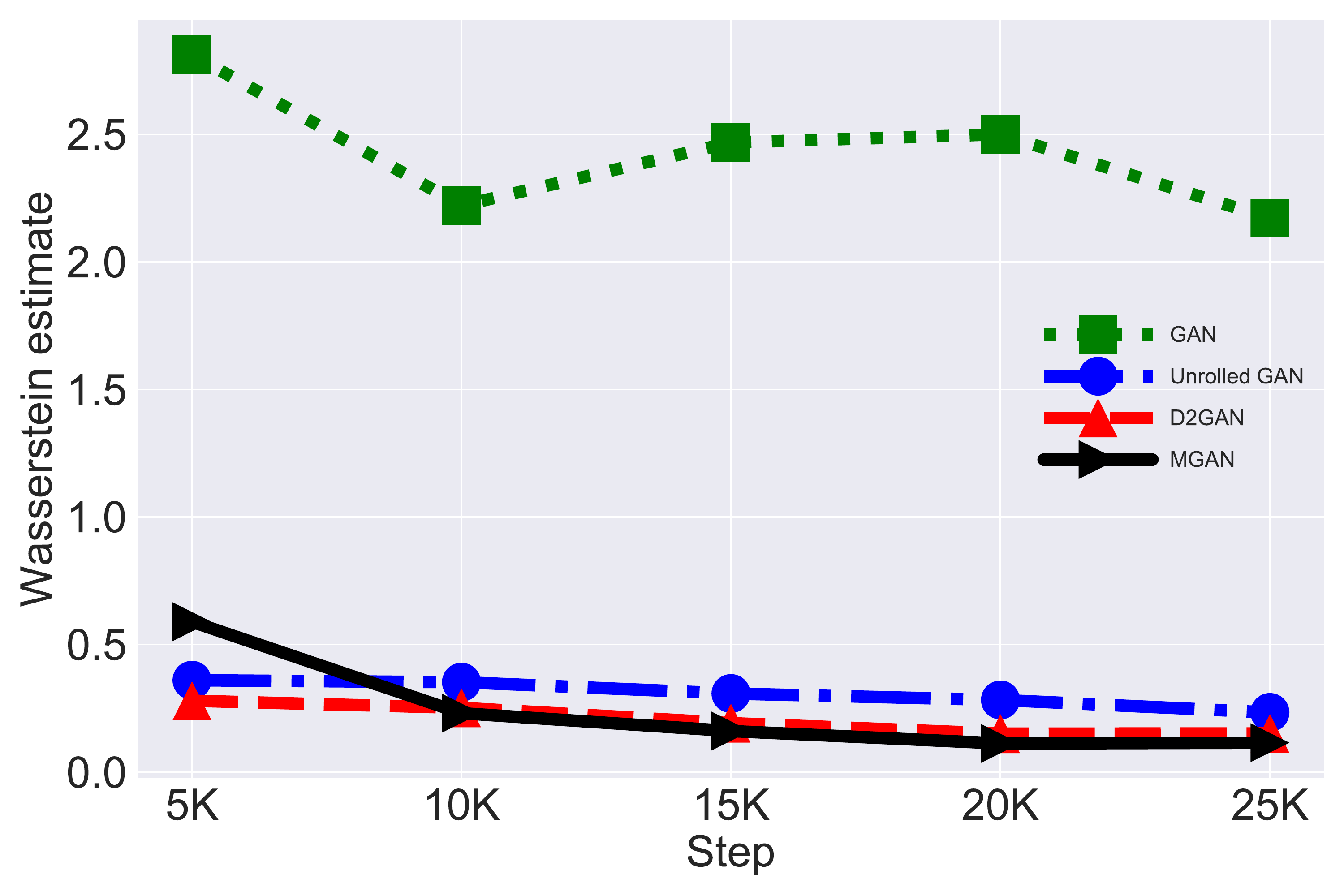}
\par\end{centering}

\noindent \centering{}}
\par\end{center}%
\end{minipage}\hfill{}%
\begin{minipage}[t]{0.67\columnwidth}%
\noindent \begin{center}
\subfloat[Evolution of data (in blue) generated by GAN, UnrolledGAN, D2GAN and
our $\protect\model$ from the top row to the bottom, respectively.
Data sampled from the true mixture of 8 Gaussians are red.\label{fig:2D_distributions}]{\noindent \centering{}\includegraphics[width=1\textwidth]{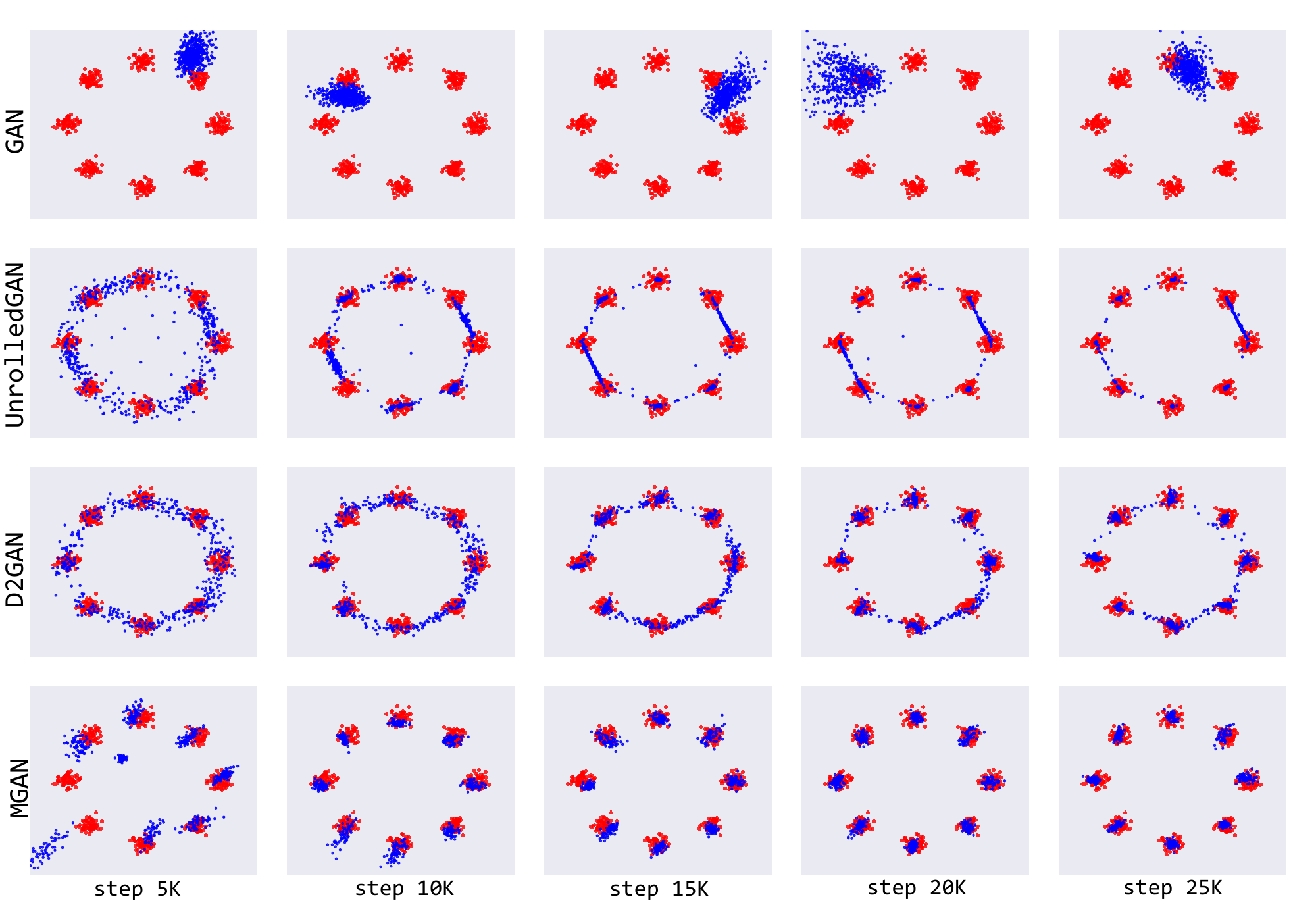}}
\par\end{center}%
\end{minipage}
\par\end{centering}

\noindent \begin{centering}
\vspace{-2mm}
\caption{The comparison of our $\protect\model$ and GAN's variants on 2D synthetic
dataset.}
\vspace{-5mm}

\par\end{centering}

\end{figure}

Lastly, we perform experiments with different numbers of generators.
The MGAN models with 2, 3, 4 and 10 generators all successfully explore
8 modes but the models with more generators generate fewer points
scattered between adjacent modes. We also examine the behavior of
the diversity coefficient $\beta$ by training the 4-generator model
with different values of $\beta$. Without the JSD force ($\beta=0$),
generated samples cluster around one mode. When $\beta=0.25$, the
JSD force is weak and generated data cluster near 4 different modes.
When $\beta=0.75$ or $1.0$, the JSD force is too strong and causes
the generators to collapse, generating 4 increasingly tight clusters.
When $\beta=0.5$, generators successfully cover all of the 8 modes.
Please refer to Appendix~\ref{sec:Synthetic-data-appendix} for experimental
details.

\subsection{Real-world Datasets\label{Real-world-Datasets}}

Next we train our proposed method on real-world databases from natural
scenes to investigate its performance and scalability on much more
challenging large-scale image data.

\paragraph{Datasets.}

We use 3 widely-adopted datasets: CIFAR-10 \citep{krizhevsky2009learning},
STL-10 \citep{coates2011analysis} and ImageNet \citep{russakovsky2015imagenet}.
CIFAR-10 contains 50,000 32$\times$32 training images of 10 classes:
airplane, automobile, bird, cat, deer, dog, frog, horse, ship, and
truck. STL-10, subsampled from ImageNet, is a more diverse dataset
than CIFAR-10, containing about 100,000 96$\times$96 images. ImageNet
(2012 release) presents the largest and most diverse consisting of
over 1.2 million images from 1,000 classes. In order to facilitate
fair comparison with the baselines in \citep{warde2016improving,tu_etal_nips17_d2gan},
we follow the procedure of \citep{krizhevsky2012imagenet} to resize
the STL-10 and ImageNet images down to 48$\times$48 and 32$\times$32,
respectively.

\paragraph{Evaluation protocols.}

For quantitative evaluation, we adopt the Inception score proposed
in \citep{salimans2016improved}, which computes $\exp\left(\mathbb{E}_{\mathbf{x}}\left[KL\left(p\left(y\vert\mathbf{x}\right)\Vert p\left(y\right)\right)\right]\right)$
where $p\left(y\vert\mathbf{x}\right)$ is the conditional label distribution
for the image $\mathbf{x}$ estimated by the reference Inception model
\citep{szegedy2015going}. This metric rewards good and varied samples
and is found to be well-correlated with human judgment \citep{salimans2016improved}.
We use the code provided in \citep{salimans2016improved} to compute
the Inception scores for 10 partitions of 50,000 randomly generated
samples. For qualitative demonstration of image quality obtained by
our proposed model, we show samples generated by the mixture as well
as samples produced by each generator. Samples are randomly drawn
rather than cherry-picked.

\paragraph{Model architectures.}

Our generator and discriminator architectures closely follow the DCGAN's
design \citep{radford2015unsupervised}. The only difference is we
apply batch normalization \citep{ioffe2015batch} to all layers in
the networks except for the output layer. Regarding the classifier,
we empirically find that our proposed $\model$ achieves the best
performance (i.e., fast convergence rate and high inception score)
when the classifier shares parameters of all layers with the discriminator
except for the output layer. The reason is that this parameter sharing
scheme would allow the classifier and discriminator to leverage their
common features and representations learned at every layer, thus helps
to improve and speed up the training progress. When the parameters
are not tied, the model learns slowly and eventually yields lower
performance.

During training we observe that the percentage of active neurons chronically
declined (see Appendix~\ref{sec:Real-world-datasets_appendix}).
One possible cause is that the batch normalization center (offset)
is gradually shifted to the negative range, thus deactivating up to
45\% of ReLU units of the generator networks. Our ad-hoc solution
for this problem is to fix the offset at zero for all layers in the
generator networks. The rationale is that for each feature map, the
ReLU gates will open for about 50\% highest inputs in a minibatch
across all locations and generators, and close for the rest.

We also experiment with other activation functions of generator networks.
First we use Leaky ReLU and obtain similar results with using ReLU.
Then we use MaxOut units \citep{goodfellow2013maxout} and achieves
good Inception scores but generates unrecognizable samples. Finally,
we try SeLU \citep{klambauer2017self} but fail to train our model.

\paragraph{Hyperparameters.}

Three key hyperparameters of our model are: number of generators $\K$,
coefficient $\beta$ controlling the diversity and the minibatch size.
We use a minibatch size of $\left[\nicefrac{128}{\K}\right]$ for
each generator, so that the total number of samples for training all
generators is about 128. We train models with 4 generators and 10
generators corresponding with minibatch sizes of 32 and 12 each, and
find that models with 10 generators performs better. For ImageNet,
we try an additional setting with 32 generators and a minibatch size
of 4 for each. The batch of 4 samples is too small for updating sufficient
statistics of a batch-norm layer, thus we drop batch-norm in the input
layer of each generator. This 32-generator model, however, does not
obtain considerably better results than the 10-generator one. Therefore
in what follows we only report the results of models with 10 generators.
For the diversity coefficient $\beta$, we observe no significant
difference in Inception scores when varying the value of $\beta$
but the quality of generated images declines when $\beta$ is too
low or too high. Generated samples by each generator vary more when
$\beta$ is low, and vary less but become less realistic when $\beta$
is high. We find a reasonable range for $\beta$ to be $\left(0.01,1.0\right)$,
and finally set to 0.01 for CIFAR-10, 0.1 for ImageNet and 1.0 for
STL-10.

\paragraph{Inception results.}

We now report the Inception scores obtained by our $\model$ and baselines
in Tab.~\ref{tab:inception_scores}. It is worthy to note that only
models trained in a completely unsupervised manner \emph{without label}
information are included for fair comparison; and DCGAN's and D2GAN's
results on STL-10 are available only for the models trained on 32$\times$32
resolution. Overall, our proposed model outperforms the baselines
by large margins and achieves state-of-the-art performance on all
datasets. Moreover, we would highlight that our $\model$ obtains
a score of 8.33 on CIFAR-10 that is even better than those of models
trained \emph{with labels} such as 8.09 of Improved GAN \citep{salimans2016improved}
and 8.25 of AC-GAN \citep{odena2016conditional}. In addition, we
train our model on the original 96$\times$96 resolution of STL-10
and achieve a score of 9.79$\pm$0.08. This suggests the $\model$
can be successfully trained on higher resolution images and achieve
the higher Inception score.

\begin{table}[h]
\noindent \centering{}\caption{Inception scores on different datasets. ``--'' denotes unavailable
result.\label{tab:inception_scores}}
\begin{tabular}{lrrr}
\hline 
Model  & CIFAR-10 & STL-10 & ImageNet\tabularnewline
\hline 
Real data  & 11.24$\pm$0.16 & 26.08$\pm$0.26 & 25.78$\pm$0.47\tabularnewline
\rowcolor{even_color}WGAN \citep{arjovsky_etal_arxiv17_wasserstein_gan} & 3.82$\pm$0.06 & --~~~~~~~~~~~ & --~~~~~~~~~~~\tabularnewline
MIX+WGAN \citep{arora2017generalization} & 4.04$\pm$0.07 & --~~~~~~~~~~~ & --~~~~~~~~~~~\tabularnewline
\rowcolor{even_color}Improved-GAN \citep{salimans2016improved} & 4.36$\pm$0.04 & --~~~~~~~~~~~ & --~~~~~~~~~~~\tabularnewline
ALI \citep{dumoulin2016adversarially} & 5.34\textbf{\emph{$\pm$}}0.05 & --~~~~~~~~~~~ & --~~~~~~~~~~~\tabularnewline
\rowcolor{even_color}BEGAN \citep{berthelot2017began} & 5.62~~~~~~~~~~ & --~~~~~~~~~~~ & --~~~~~~~~~~~\tabularnewline
MAGAN \citep{wang2017magan} & 5.67 ~~~~~~~~~ & --~~~~~~~~~~~ & --~~~~~~~~~~~\tabularnewline
\rowcolor{even_color}GMAN \citep{durugkar2016generative} & 6.00$\pm$0.19 & --~~~~~~~~~~~ & --~~~~~~~~~~~\tabularnewline
DCGAN \citep{radford2015unsupervised} & 6.40$\pm$0.05 & 7.54~~~~~~~~~~ & 7.89~~~~~~~~~~~\tabularnewline
\rowcolor{even_color}DFM \citep{warde2016improving} & 7.72$\pm$0.13 & 8.51$\pm$0.13 & 9.18$\pm$0.13\tabularnewline
D2GAN \citep{tu_etal_nips17_d2gan} & 7.15$\pm$0.07 & 7.98~~~~~~~~~~ & 8.25~~~~~~~~~~~\tabularnewline
\rowcolor{even_color}\textbf{$\boldsymbol{\model}$} & \textbf{8.33$\pm$0.10} & \textbf{9.22$\pm$0.11} & \textbf{9.32$\pm$0.10}\tabularnewline
\hline 
\end{tabular}
\end{table}

\paragraph{Image generation.}

Next we present samples randomly generated by our proposed model trained
on the 3 datasets for qualitative assessment. Fig.~\ref{fig:exp_cifar10_samples}
shows CIFAR-10 32$\times$32 images containing a wide range of objects
in such as airplanes, cars, trucks, ships, birds, horses or dogs.
Similarly, STL-10 48$\times$48 generated images in Fig.~\ref{fig:exp_stl10_48x48_samples}
include cars, ships, airplanes and many types of animals, but with
wider range of different themes such as sky, underwater, mountain
and forest. Images generated for ImageNet 32$\times$32 are diverse
with some recognizable objects such as lady, old man, birds, human
eye, living room, hat, slippers, to name a few. Fig.~\ref{fig:STL96_good}
shows several cherry-picked STL-10 96$\times$96 images, which demonstrate
that the $\model$ is capable of generating visually appealing images
with complicated details. However, many samples are still incomplete
and unrealistic as shown in Fig.~\ref{fig:STL96_bad}, leaving plenty
of room for improvement.

\begin{figure}[h]
\noindent \begin{centering}
\subfloat[CIFAR-10 32$\times$32.\label{fig:exp_cifar10_samples}]{\noindent \begin{centering}
\includegraphics[width=0.33\textwidth]{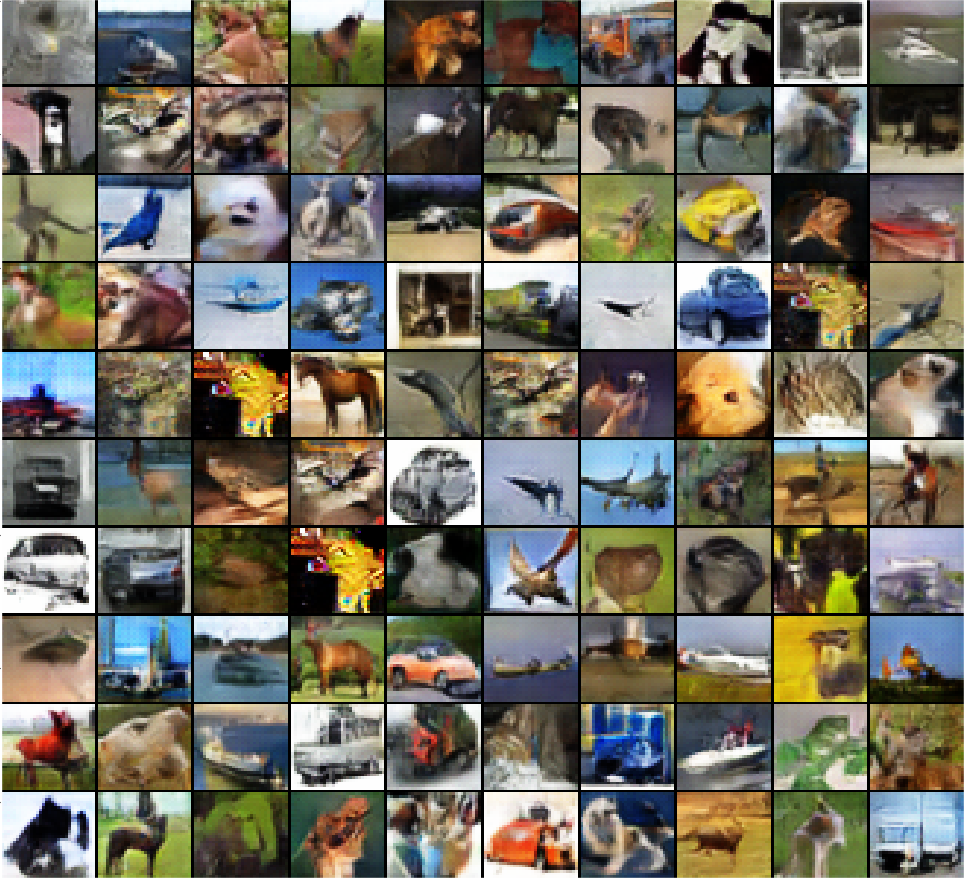}
\par\end{centering}

\noindent \centering{}}\subfloat[STL-10 48$\times$48.\label{fig:exp_stl10_48x48_samples}]{\noindent \begin{centering}
\includegraphics[width=0.3\textwidth]{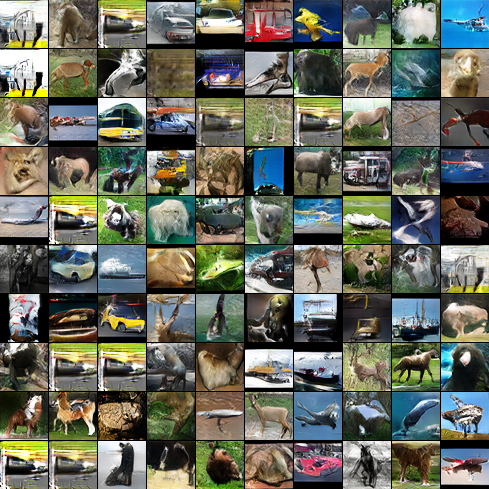}
\par\end{centering}

\noindent \centering{}}\subfloat[ImageNet 32$\times$32.\label{fig:exp_imagenet_samples}]{\noindent \begin{centering}
\includegraphics[width=0.3\textwidth]{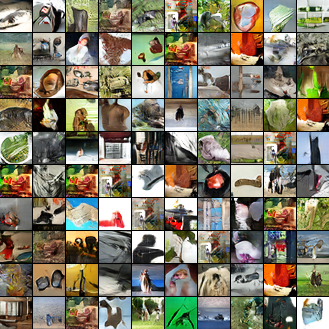}
\par\end{centering}

\noindent \centering{}}
\par\end{centering}

\noindent \centering{}\caption{Images generated by our proposed $\protect\model$ trained on natural
image datasets. Due to the space limit, please refer to the appendix
for larger plots.\label{fig:32x32_images}}
\end{figure}

\begin{figure}[h]
\noindent \begin{centering}
\subfloat[Cherry-picked samples.\label{fig:STL96_good}]{\noindent \centering{}\includegraphics[width=0.45\textwidth]{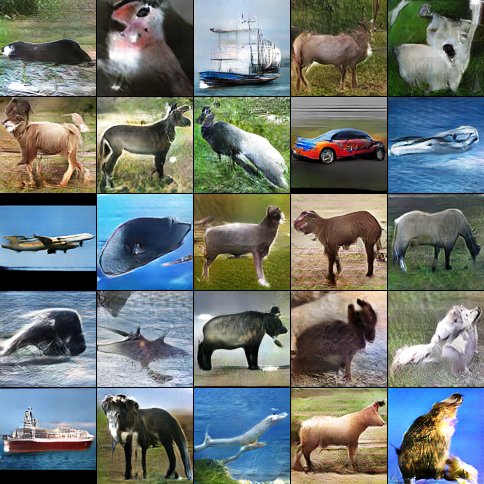}}\hfill{}\subfloat[Incomplete, unrealistic samples.\label{fig:STL96_bad}]{\noindent \begin{centering}
\includegraphics[width=0.45\textwidth]{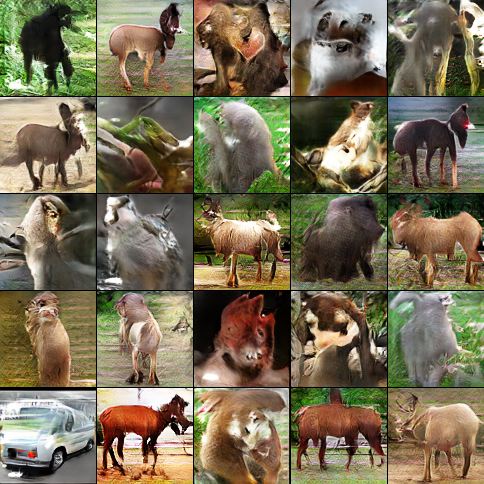}
\par\end{centering}

\noindent \centering{}}
\par\end{centering}

\noindent \centering{}\caption{Images generated by our $\protect\model$ trained on the original
96$\times$96 STL10 dataset.\label{STL96}}
\end{figure}

Finally, we investigate samples generated by each generator as well
as the evolution of these samples through numbers of training epochs.
Fig.~\ref{fig:cifar_samples_by_gens} shows images generated by each
of the 10 generators in our $\model$ trained on CIFAR-10 at epoch
20, 50, and 250 of training. Samples in each row correspond to a different
generator. Generators start to specialize in generating different
types of objects as early as epoch 20 and become more and more consistent:
generator 2 and 3 in flying objects (birds and airplanes), generator
4 in full pictures of cats and dogs, generator 5 in portraits of cats
and dogs, generator 8 in ships, generator 9 in car and trucks, and
generator 10 in horses. Generator 6 seems to generate images of frog
or animals in a bush. Generator 7, however, collapses in epoch 250.
One possible explanation for this behavior is that images of different
object classes tend to have different themes. Lastly, \citet{wang2016ensembles}
noticed one of the causes for non-convergence in GANs is that the
generators and discriminators constantly vary; the generators at two
consecutive epochs of training generate significantly different images.
This experiment demonstrates the effect of the JSD force in preventing
generators from moving around the data space.

\begin{figure}[h]
\noindent \begin{centering}
\subfloat[Epoch \#20.]{\noindent \centering{}\includegraphics[width=0.3\textwidth]{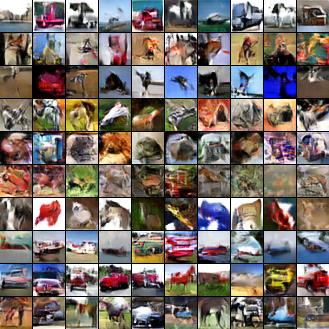}}\enskip{}\subfloat[Epoch \#50.]{\noindent \centering{}\includegraphics[width=0.3\textwidth]{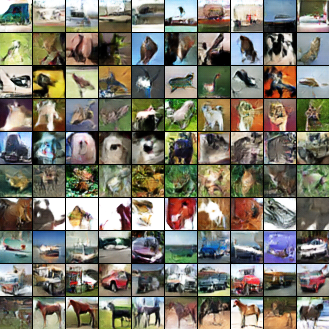}}\enskip{}\subfloat[Epoch \#250.]{\noindent \centering{}\includegraphics[width=0.3\textwidth]{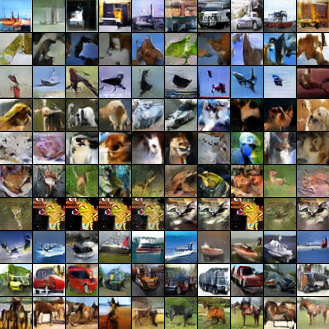}}
\par\end{centering}

\noindent \centering{}\caption{Images generated by our $\protect\model$ trained on CIFAR10 at different
epochs. Samples in each row from the top to the bottom correspond
to a different generator.\label{fig:cifar_samples_by_gens}}
\end{figure}

\section{Conclusion}

We have presented a novel adversarial model to address the mode collapse
in GANs. Our idea is to approximate data distribution using a mixture
of multiple distributions wherein each distribution captures a subset
of data modes separately from those of others. To achieve this goal,
we propose a minimax game of one discriminator, one classifier and
many generators to formulate an optimization problem that minimizes
the JSD between $P_{data}$ and $P_{model}$, i.e., a mixture of distributions
induced by the generators, whilst maximizes JSD among such generator
distributions. This helps our model generate diverse images to better
cover data modes, thus effectively avoids mode collapse. We term our
proposed model \emph{Mixture Generative Adversarial Network} ($\model$).

The $\model$ can be efficiently trained by sharing parameters between
its discriminator and classifier, and among its generators, thus our
model is scalable to be evaluated on real-world large-scale datasets.
Comprehensive experiments on synthetic 2D data, CIFAR-10, STL-10 and
ImageNet databases demonstrate the following capabilities of our model:
(i) achieving state-of-the-art Inception scores; (ii) generating diverse
and appealing recognizable objects at different resolutions; and (iv)
specializing in capturing different types of objects by the generators.

\bibliographystyle{iclr18}

\appendix

\section{Appendix: Framework\label{sec:appendix_framework}}

In our proposed method, generators $G_{1}$, $G_{2}$, ... $G_{\K}$
are deep convolutional neural networks parameterized by $\btheta_{G}$.
These networks share parameters in all layers except for the input
layers. The input layer for generator $G_{k}$ is parameterized by
the mapping $f_{\mathbf{\btheta}_{G},k}\left(\mathbf{z}\right)$ that
maps the sampled noise $\mathbf{z}$ to the first hidden layer activation
$\mathbf{h}$. The shared layers are parameterized by the mapping
$g_{\mathbf{\btheta}_{G}}\left(\mathbf{h}\right)$ that maps the first
hidden layer to the generated data. The pseudo-code of sampling from
the mixture is described in Alg.~\ref{algo:sampling_from_mixture}.
Classifier $C$ and classifier $D$ are also deep convolutional neural
networks that are both parameterized by $\btheta_{CD}$. They share
parameters in all layers except for the last layer. The pseudo-code
of alternatively learning $\btheta_{G}$ and $\btheta_{CD}$ using
stochastic gradient descend is described in Alg.~\ref{algo:training}.

\begin{algorithm}
\begin{algor}[1]
\item [{{*}}] Sample noise $\mathbf{z}$ from the prior $P_{\mathbf{z}}$.
\item [{{*}}] Sample a generator index $u$ from $\textnormal{Mult}\left(\pi_{1},\pi_{2},...,\pi_{\K}\right)$
with predefined mixing probability $\bpi=\left(\pi_{1},\pi_{2},...,\pi_{\K}\right)$.
\item [{{*}}] $\mathbf{h}=f_{\mathbf{\btheta}_{G},u}\left(\mathbf{z}\right)$
\item [{{*}}] $\mathbf{x}=g_{\mathbf{\btheta}_{G}}\left(\mathbf{h}\right)$
\item [{{*}}] Return generated data $\mathbf{x}$ and the index $u$.
\end{algor}
\caption{Sampling from $\protect\model$'s mixture of generators.\label{algo:sampling_from_mixture}}
\end{algorithm}

\begin{algorithm}
\begin{algor}[1]
\item [{for}] number of training iterations
\item [{{*}}] Sample a minibatch of $\M$ data points $\left(\mathbf{x}^{\left(1\right)},\mathbf{x}^{\left(2\right)},...,\mathbf{x}^{\left(\M\right)}\right)$
from the data distribution $P_{data}$.
\item [{{*}}] Sample a minibatch of $\N$ generated data points $\left(\mathbf{x^{\prime}}^{\left(1\right)},\mathbf{x^{\prime}}^{\left(2\right)},...,\mathbf{x^{\prime}}^{\left(\N\right)}\right)$
and $\N$ indices $\left(u_{1},u_{2},...,u_{\N}\right)$ from the
current mixture.
\item [{{*}}] $\mathcal{L}_{C}=-\frac{1}{\N}\sum_{n=1}^{\N}\log C_{u_{n}}\left(\mathbf{x^{\prime}}^{\left(n\right)}\right)$
\item [{{*}}] $\mathcal{L}_{D}=-\frac{1}{\M}\sum_{m=1}^{\M}\log D\left(\mathbf{x}^{\left(m\right)}\right)-\frac{1}{\N}\sum_{n=1}^{\N}\log\left[1-D\left(\mathbf{x^{\prime}}^{\left(n\right)}\right)\right]$
\item [{{*}}] Update classifier $C$ and discriminator $D$ by descending
along their gradient: $\nabla_{\mathbf{\btheta}_{CD}}\left(\mathcal{L}_{C}+\mathcal{L}_{D}\right)$.
\item [{{*}}] Sample a minibatch of $\N$ generated data points $\left(\mathbf{x^{\prime}}^{\left(1\right)},\mathbf{x^{\prime}}^{\left(2\right)},...,\mathbf{x^{\prime}}^{\left(\N\right)}\right)$
and $\N$ indices $\left(u_{1},u_{2},...,u_{\N}\right)$ from the
current mixture.
\item [{{*}}] $\mathcal{L}_{G}=-\frac{1}{\N}\sum_{n=1}^{\N}\log D\left(\mathbf{x^{\prime}}^{\left(n\right)}\right)-\frac{\beta}{\N}\sum_{n=1}^{\N}\log C_{u_{n}}\left(\mathbf{x^{\prime}}^{\left(n\right)}\right)$
\item [{{*}}] Update the mixture of generators $G$ by ascending along
its gradient: $\nabla_{\mathbf{\btheta}_{G}}\mathcal{L_{G}}$.
\item [{endfor}]~
\end{algor}
\caption{Alternative training of $\protect\model$ using stochastic gradient
descent.\label{algo:training}}

\end{algorithm}

\section{Appendix: Proofs for Section~\ref{theoretical_analysis}\label{sec:appendix_proofs}}

\paragraph{Proposition 1 (Prop.~\ref{Proposition1} restated).}

\emph{For fixed generators $G_{1}$, $G_{2}$, ..., $G_{\K}$ and
mixture weights $\pi_{1},\pi_{2},...,\pi_{\K}$, the optimal classifier
$C^{*}=C_{1:\K}^{*}$ and discriminator $D^{*}$ for $\mathcal{J}\left(G,C,D\right)$
are:
\begin{align*}
C_{k}^{*}\left(\mathbf{x}\right) & =\frac{\pi_{k}p_{G_{k}}\left(\mathbf{x}\right)}{\sum_{j=1}^{K}\pi_{j}p_{G_{j}}\left(\mathbf{x}\right)}\\
D^{*}\left(\mathbf{x}\right) & =\frac{p_{data}\left(\mathbf{x}\right)}{p_{data}\left(\mathbf{x}\right)+p_{model}\left(\mathbf{x}\right)}
\end{align*}
}
\begin{proof}
The optimal $D^{*}$ was proved in Prop.~1 in \citep{goodfellow2016nips}.
This section shows a similar proof for the optimal $C^{*}$. Assuming
that $C^{*}$ can be optimized in the functional space, we can calculate
the functional derivatives of $\mathcal{J}\left(G,C,D\right)$with
respect to each $C_{k}\left(\mathbf{x}\right)$ for $k\in\left\{ 2,...,\K\right\} $
and set them equal to zero:
\begin{align}
\frac{\delta\mathcal{J}}{\delta C_{k}\left(\mathbf{x}\right)} & =-\beta\frac{\delta}{\delta C_{k}\left(\mathbf{x}\right)}\int\left(\pi_{1}p_{G_{1}}\left(\mathbf{x}\right)\log\left(1-\sum_{k=2}^{\K}C_{k}\left(\mathbf{x}\right)\right)+\sum_{k=2}^{\K}\pi_{k}p_{G_{k}}\left(\mathbf{x}\right)\log C_{k}\left(\mathbf{x}\right)\right)\d x\nonumber \\
 & =-\beta\left(\frac{\pi_{k}p_{G_{k}}\left(\mathbf{x}\right)}{C_{k}\left(\mathbf{x}\right)}-\frac{\pi_{1}p_{G_{1}}\left(\mathbf{x}\right)}{C_{1}\left(\mathbf{x}\right)}\right)\label{C_functional_derivative}
\end{align}
Setting $\frac{\delta\mathcal{J}\left(G,C,D\right)}{\delta C_{k}\left(\mathbf{x}\right)}$
to 0 for $k\in\left\{ 2,...,\K\right\} $, we get:
\begin{equation}
\frac{\pi_{1}p_{G_{1}}\left(\mathbf{x}\right)}{C_{1}^{*}\left(\mathbf{x}\right)}=\frac{\pi_{2}p_{G_{2}}\left(\mathbf{x}\right)}{C_{2}^{*}\left(\mathbf{x}\right)}=...=\frac{\pi_{K}p_{G_{K}}\left(\mathbf{x}\right)}{C_{\K}^{*}\left(\mathbf{x}\right)}\label{C_zero_gradient}
\end{equation}
$C_{k}^{*}\left(\mathbf{x}\right)=\frac{\pi_{k}p_{G_{k}}\left(\mathbf{x}\right)}{\sum_{j=1}^{\K}\pi_{j}p_{G_{j}}\left(\mathbf{x}\right)}$
results from Eq.~(\ref{C_zero_gradient}) due to the fact that $\sum_{k=1}^{\K}C_{k}^{*}\left(\mathbf{x}\right)=1$.
\end{proof}

\paragraph{Reformulation of $\mathcal{L}\left(G_{1:\protect\K}\right)$.}

Replacing the optimal $C^{*}$ and $D^{*}$ into Eq.~(\ref{game_formula}),
we can reformulate the objective function for the generator as follows:
\begin{alignat}{1}
\mathcal{\mathcal{L}}\left(G_{1:\K}\right) & =\mathcal{J}\left(G,C^{*},D^{*}\right)\nonumber \\
 & =\mathbb{E}_{\bx\sim P_{data}}\left[\log\frac{p_{data}\left(\mathbf{x}\right)}{p_{data}\left(\mathbf{x}\right)+p_{model}\left(\mathbf{x}\right)}\right]+\mathbb{E}_{\bx\sim P_{model}}\left[\log\frac{p_{model}\left(\mathbf{x}\right)}{p_{data}\left(\mathbf{x}\right)+p_{model}\left(\mathbf{x}\right)}\right]\nonumber \\
 & \,\,\,\,\,\,-\beta\left\{ \sum_{k=1}^{\K}\pi_{k}\mathbb{E}_{\bx\sim P_{G_{k}}}\left[\log\frac{\pi_{k}p_{G_{k}}\left(\mathbf{x}\right)}{\sum_{j=1}^{\K}\pi_{j}p_{G_{j}}\left(\mathbf{x}\right)}\right]\right\} \label{Lg-1}
\end{alignat}
The sum of the first two terms in Eq.~(\ref{Lg-1}) was shown in
\citep{goodfellow2014generative} to be $2\cdot\textnormal{JSD}\left(P_{data}\Vert P_{model}\right)-\log4$.
The last term $\beta\{*\}$ of Eq.~(\ref{Lg-1}) is related to the
JSD for the $\K$ distributions:
\begin{align}
* & =\sum_{k=1}^{\K}\pi_{k}\mathbb{E}_{\bx\sim P_{G_{k}}}\left[\log\frac{\pi_{k}p_{G_{k}}\left(\mathbf{x}\right)}{\sum_{j=1}^{\K}\pi_{j}p_{G_{j}}\left(\mathbf{x}\right)}\right]\nonumber \\
 & =\sum_{k=1}^{\K}\pi_{k}\mathbb{E}_{\bx\sim P_{G_{k}}}\left[\log p_{G_{k}}\left(\mathbf{x}\right)\right]-\sum_{k=1}^{\K}\pi_{k}\mathbb{E}_{\bx\sim P_{G_{k}}}\left[\log\sum_{j=1}^{\K}\pi_{j}p_{G_{j}}\left(\mathbf{x}\right)\right]+\sum_{k=1}^{\K}\pi_{k}\log\pi_{k}\nonumber \\
 & =-\sum_{k=1}^{\K}\pi_{k}\mathbb{H}\left(p_{G_{k}}\right)+\mathbb{H}\left(\sum_{j=1}^{\K}\pi_{j}p_{G_{j}}\left(\mathbf{x}\right)\right)+\sum_{k=1}^{\K}\pi_{k}\log\pi_{k}\nonumber \\
 & =\textnormal{JSD}_{\bpi}\left(P_{G_{1}},P_{G_{2}},...,P_{G_{\K}}\right)+\sum_{k=1}^{\K}\pi_{k}\log\pi_{k}\label{Lg_JSD}
\end{align}
where $\mathbb{H}\left(P\right)$ is the Shannon entropy for distribution
$P$. Thus, $\mathcal{\mathcal{L}}\left(G_{1:\K}\right)$ can be rewritten
as:
\begin{align*}
\mathcal{\mathcal{L}}\left(G_{1:\K}\right) & =-\log4+2\cdot\textnormal{JSD}\left(P_{data}\Vert P_{model}\right)-\beta\cdot\textnormal{JSD}_{\bpi}\left(P_{G_{1}},P_{G_{2}},...,P_{G_{\K}}\right)-\beta\sum_{k=1}^{\K}\pi_{k}\log\pi_{k}
\end{align*}

\paragraph{Theorem 3 (Thm.~\ref{theorem_global_minimum} restated).}

\emph{If the data distribution has the form: $p_{data}\left(\bx\right)=\sum_{k=1}^{K}\pi_{k}q_{k}\left(\bx\right)$
where the mixture components $q_{k}\left(\bx\right)$(s) are well-separated,
the minimax problem in Eq.~(\ref{game_formula}) or the optimization
problem in Eq. (\ref{eq:optimal}) has the following solution:
\[
p_{G_{k}^{*}}\left(\bx\right)=q_{k}\left(\bx\right),\,\forall k=1,\dots,\K\,\,\text{and}\,\,p_{model}\left(\bx\right)=\sum_{k=1}^{\K}\pi_{k}q_{k}\left(\bx\right)=p_{data}\left(\bx\right)
\]
, and the corresponding objective value of the optimization problem
in Eq. (\ref{eq:optimal}) is $-\beta\mathbb{H}\left(\bpi\right)=-\beta\sum_{k=1}^{K}\pi_{k}\log\frac{1}{\pi_{k}}$.}
\begin{proof}
We first recap the optimization problem for finding the optimal $G^{*}$:
\[
\min_{G}\left(2\cdot\textnormal{JSD}\left(P_{data}\Vert P_{model}\right)-\beta\cdot\textnormal{JSD}_{\pi}\left(P_{G_{1}},P_{G_{2}},...,P_{G_{\K}}\right)\right)
\]
The JSD in Eq.~(\ref{Lg_JSD}) is given by:
\begin{align}
\textnormal{JSD}_{\bpi}\left(P_{G_{1}},P_{G_{2}},...,P_{G_{\K}}\right) & =\sum_{k=1}^{\K}\pi_{k}\mathbb{E}_{\bx\sim P_{G_{k}}}\left[\log\frac{\pi_{k}p_{G_{k}}\left(\mathbf{x}\right)}{\sum_{j=1}^{\K}\pi_{j}p_{G_{j}}\left(\mathbf{x}\right)}\right]-\sum_{k=1}^{\K}\pi_{k}\log\pi_{k}\label{JSD}
\end{align}
The $i$-th expectation in Eq.~(\ref{JSD}) can be derived as follows:
\begin{align*}
\mathbb{E}_{\bx\sim P_{G_{k}}}\left[\log\frac{\pi_{k}p_{G_{k}}\left(\mathbf{x}\right)}{\sum_{j=1}^{\K}\pi_{j}p_{G_{j}}\left(\mathbf{x}\right)}\right] & \leq\mathbb{E}_{\bx\sim P_{G_{k}}}\left[\log1\right]\leq0
\end{align*}
and the equality occurs if $\frac{\pi_{k}p_{G_{k}}\left(\mathbf{x}\right)}{\sum_{j=1}^{\K}\pi_{j}p_{G_{j}}\left(\mathbf{x}\right)}=1$
almost everywhere or equivalently for almost every $\bx$ except for
those in a zero measure set, we have:
\begin{equation}
p_{G_{k}}\left(\bx\right)>0\Longrightarrow p_{G_{j}}\left(\bx\right)=0,\,\forall j\neq k\label{eq:well-separated}
\end{equation}
Therefore, we obtain the following inequality:
\[
\textnormal{JSD}_{\bpi}\left(P_{G_{1}},P_{G_{2}},...,P_{G_{\K}}\right)\leq-\sum_{k=1}^{\K}\pi_{k}\log\pi_{k}=\sum_{k=1}^{\K}\pi_{k}\log\frac{1}{\pi_{k}}=\mathbb{H}\left(\bpi\right)
\]
and the equality occurs if for almost every $\bx$ except for those
in a zero measure set, we have:
\[
\forall k:\,p_{G_{k}}\left(\bx\right)>0\Longrightarrow p_{G_{j}}\left(\bx\right)=0,\,\forall j\neq k
\]
It follows that
\[
2\cdot\textnormal{JSD}\left(P_{data}\Vert P_{model}\right)-\beta\cdot\textnormal{JSD}_{\bpi}\left(P_{G_{1}},P_{G_{2}},...,P_{G_{\K}}\right)\geq0-\beta\mathbb{H}\left(\bpi\right)=-\beta\mathbb{H}\left(\bpi\right)
\]
and we peak the minimum if $p_{G_{k}}=q_{k},\,\forall k$ since this
solution satisfies both
\[
p_{model}\left(\bx\right)=\sum_{k=1}^{\K}\pi_{k}q_{k}\left(\bx\right)=p_{data}\left(\bx\right)
\]
and the conditions depicted in Eq.~(\ref{eq:well-separated}). That
concludes our proof.\end{proof}

\section{Appendix: Additional Experiments\label{sec:appendix_exps}}

\subsection{Synthetic 2D Gaussian data\label{sec:Synthetic-data-appendix}}

The true data is sampled from a 2D mixture of 8 Gaussian distributions
with a covariance matrix 0.02$\bI$ and means arranged in a circle
of zero centroid and radius 2.0. We use a simple architecture of 8
generators with two fully connected hidden layers and a classifier
and a discriminator with one shared hidden layer. All hidden layers
contain the same number of 128 ReLU units. The input layer of generators
contains 256 noise units sampled from isotropic multivariate Gaussian
distribution $\mathcal{N}\left(0,\bI\right)$. We do not use batch
normalization in any layer. We refer to Tab.~\ref{tab:synthetic2d_exp}
for more specifications of the network and hyperparameters. ``Shared''
is short for parameter sharing among generators or between the classifier
and the discriminator. Feature maps of 8/1 in the last layer for $C$
and $D$ means that two separate fully connected layers are applied
to the penultimate layer, one for $C$ that outputs $8$ logits and
another for $D$ that outputs $1$ logit.

\begin{table}[h]
\noindent \begin{centering}
\caption{Network architecture and hyperparameters for 2D Gaussian data.\label{tab:synthetic2d_exp}}

\par\end{centering}

\noindent \centering{}%
\begin{tabular}{rllc}
\hline 
Operation & Feature maps & Nonlinearity & Shared?\tabularnewline
\hline 
$G\left(\mathbf{z}\right):\mathbf{z}\sim\mathcal{N}\left(0,\mathbf{I}\right)$ & 256 &  & \tabularnewline
Fully connected & 128 & ReLU & $\times$\tabularnewline
Fully connected & 128 & ReLU & $\surd$\tabularnewline
Fully connected & 2 & Linear & $\surd$\tabularnewline
\hline 
$C\left(\mathbf{x}\right),D\left(\mathbf{x}\right)$ & 2 &  & \tabularnewline
Fully connected & 128 & Leaky ReLU & $\surd$\tabularnewline
Fully connected & 8/1 & Softmax/Sigmoid & $\times$\tabularnewline
\hline 
Number of generators & \multicolumn{3}{l}{8}\tabularnewline
Batch size for real data & \multicolumn{3}{l}{512}\tabularnewline
Batch size for each generator & \multicolumn{3}{l}{128}\tabularnewline
Number of iterations & \multicolumn{3}{l}{25,000}\tabularnewline
Leaky ReLU slope & \multicolumn{3}{l}{0.2}\tabularnewline
Learning rate & \multicolumn{3}{l}{0.0002}\tabularnewline
Regularization constants & \multicolumn{3}{l}{$\beta=0.125$}\tabularnewline
Optimizer & \multicolumn{3}{l}{Adam$\left(\beta_{1}=0.5,\beta_{2}=0.999\right)$}\tabularnewline
Weight, bias initialization & \multicolumn{3}{l}{$\mathcal{N}\left(\mu=0,\sigma=0.02\bI\right)$, $0$}\tabularnewline
\hline 
\end{tabular}
\end{table}

\paragraph{The effect of the number of generators on generated samples.}

Fig.~\ref{fig:different_num_gens} shows samples produced by $\model$s
with different numbers of generators trained on synthetic data for
25,000 epochs. The model with 1 generator behaves similarly to the
standard GAN as expected. The models with 2, 3 and 4 generators all
successfully cover 8 modes, but the ones with more generators draw
fewer points scattered between adjacent modes. Finally, the model
with 10 generators also covers 8 modes wherein 2 generators share
one mode and one generator hovering around another mode.

\begin{figure}[h]
\noindent \begin{centering}
\subfloat[1 generator.]{\noindent \begin{centering}
\includegraphics[width=0.17\textwidth]{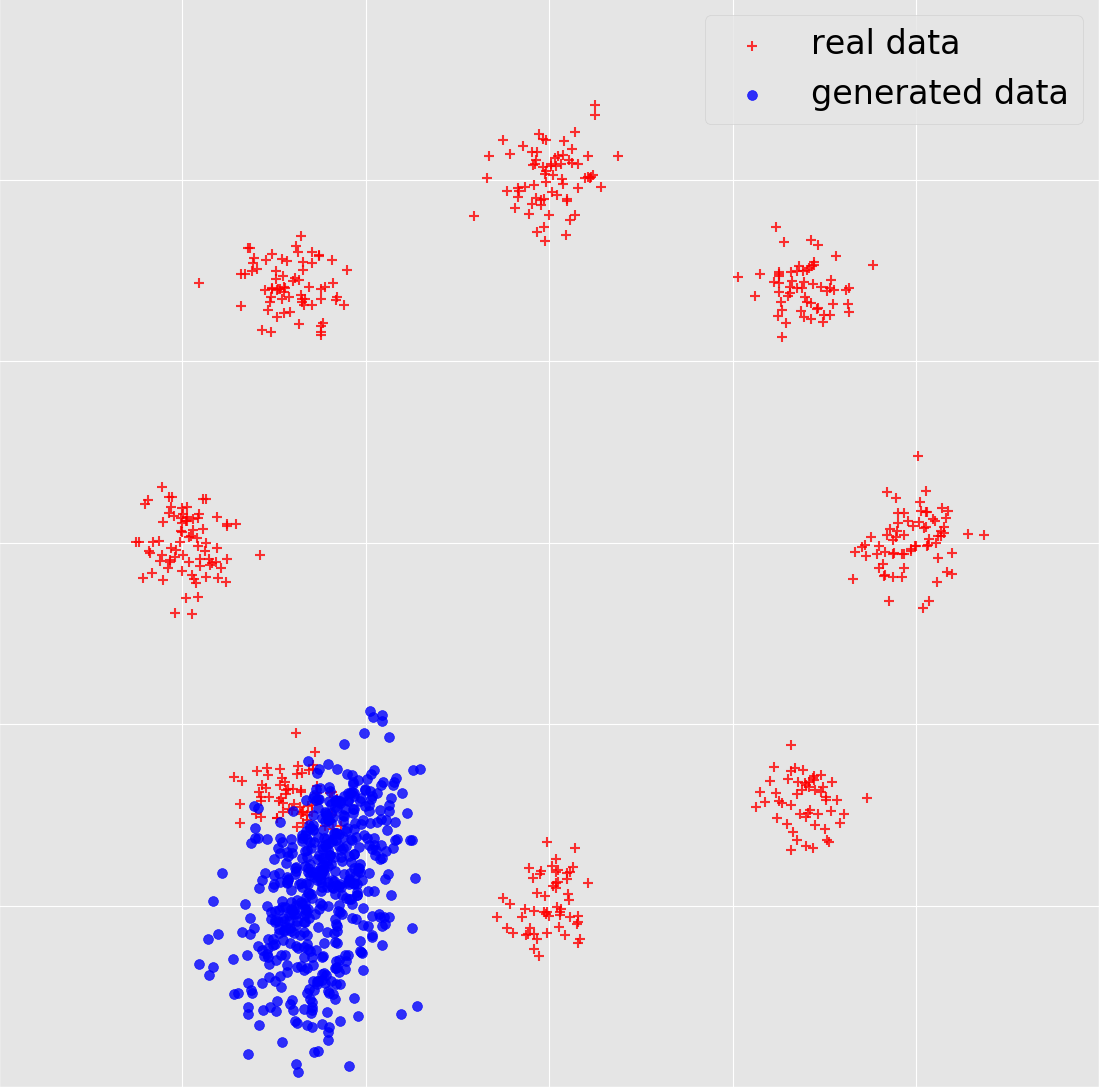} 
\par\end{centering}

\noindent \centering{}}\hfill{}\subfloat[2 generators.]{\noindent \begin{centering}
\includegraphics[width=0.17\textwidth]{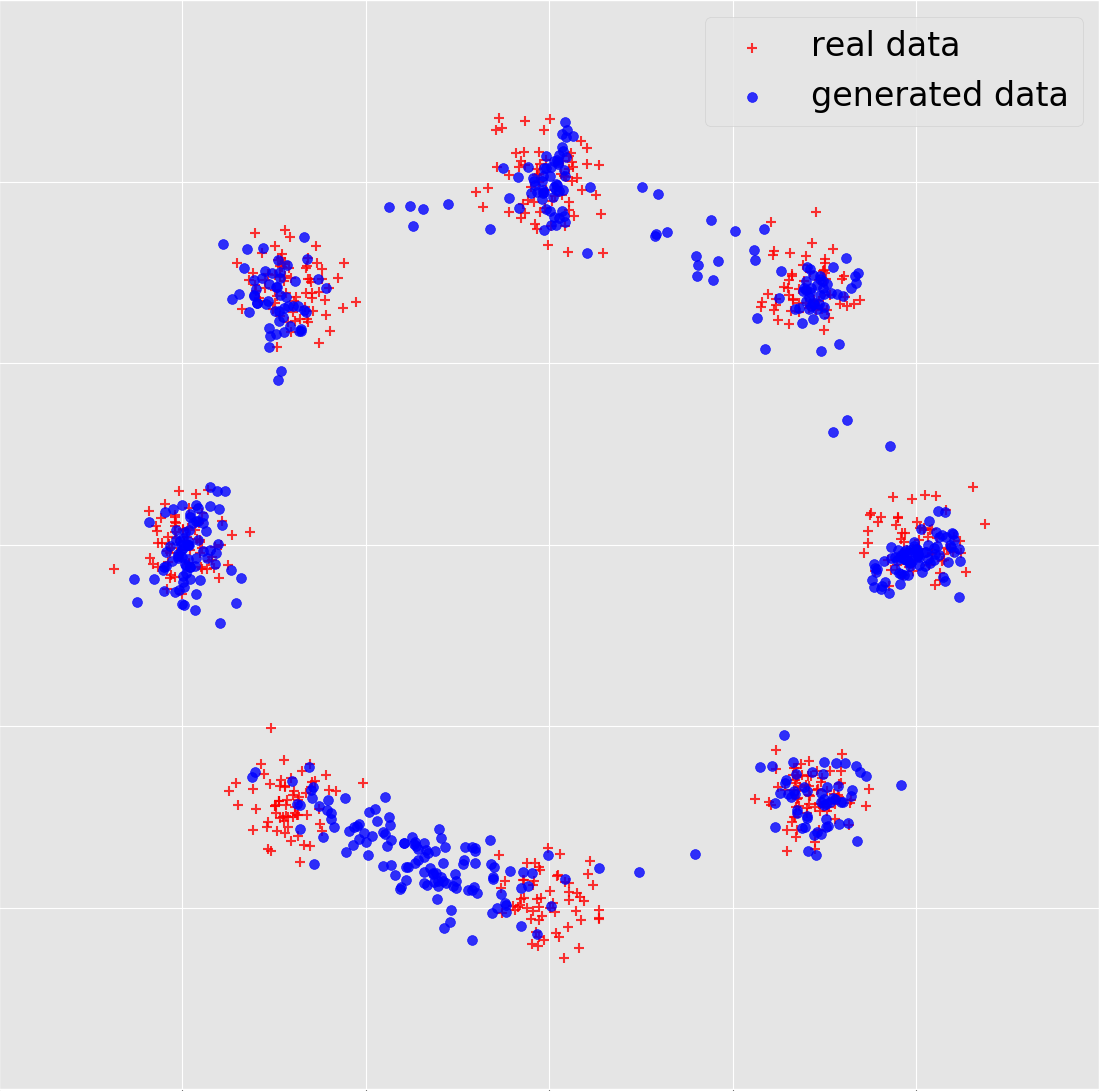} 
\par\end{centering}

\noindent \centering{}}\hfill{} \subfloat[3 generators.]{\noindent \begin{centering}
\includegraphics[width=0.17\textwidth]{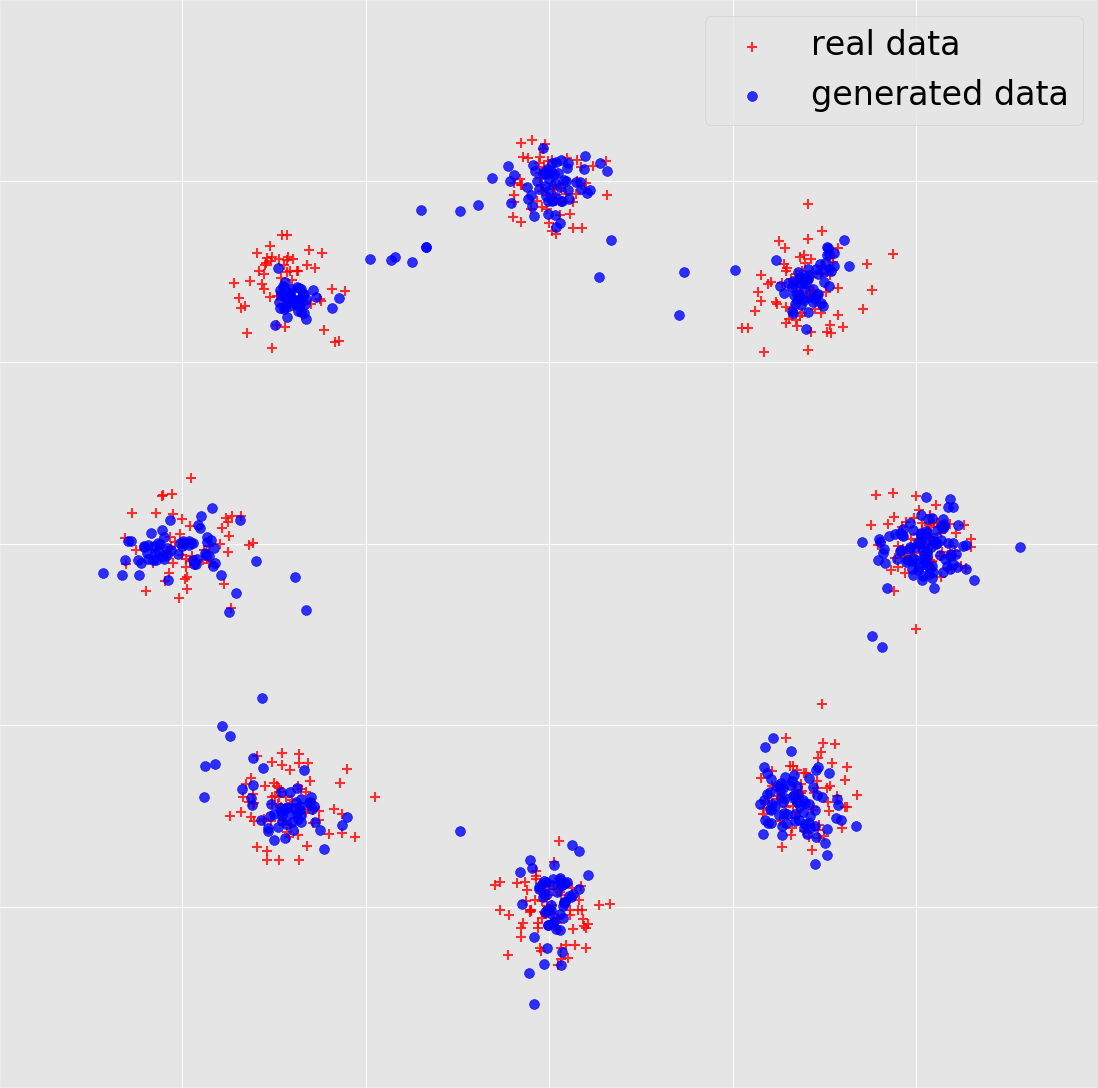}
\par\end{centering}

\noindent \centering{}}\hfill{} \subfloat[4 generators.]{\noindent \begin{centering}
\includegraphics[width=0.17\textwidth]{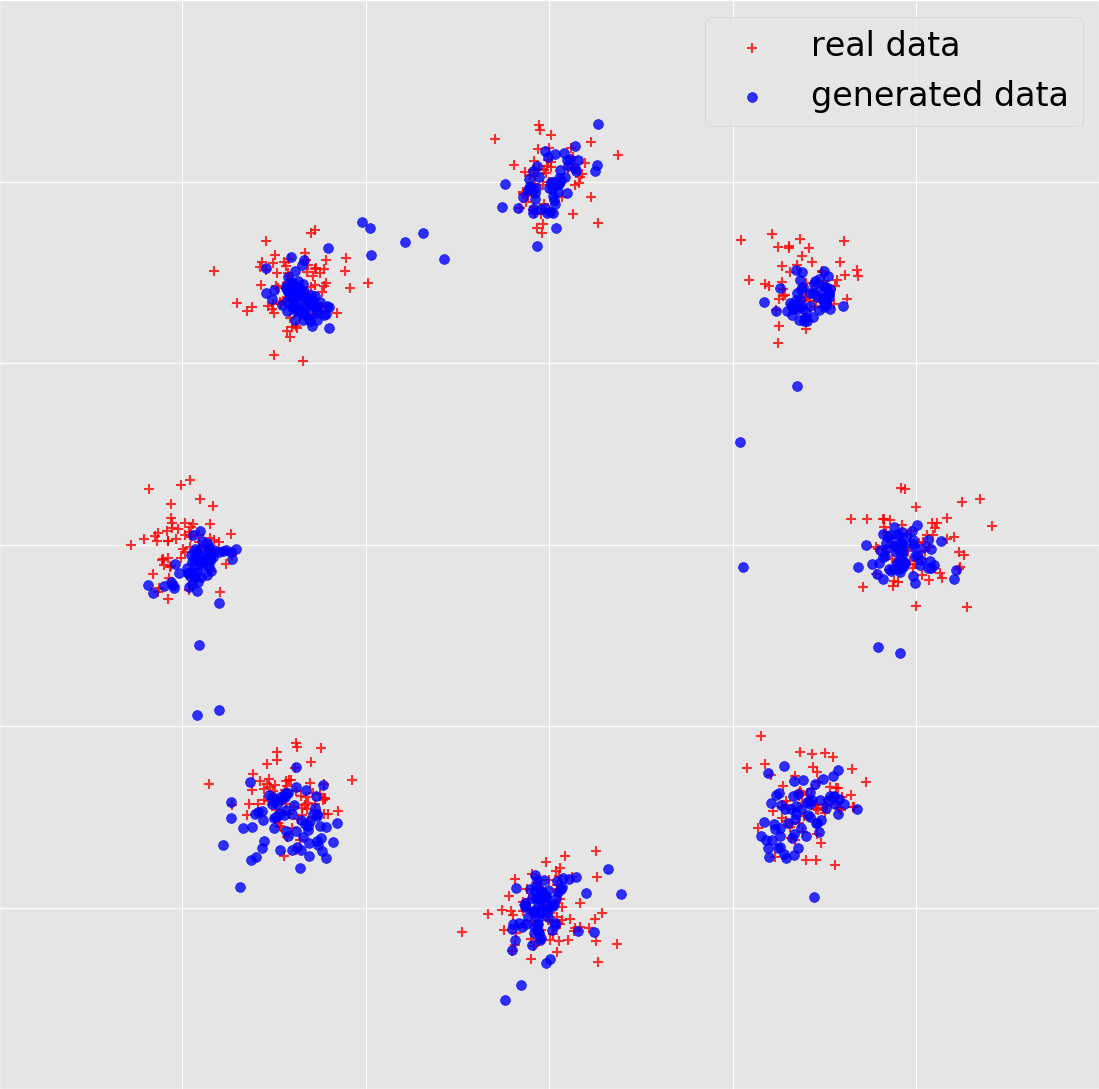}
\par\end{centering}

\noindent \centering{}}\hfill{}\subfloat[10 generators.]{\noindent \begin{centering}
\includegraphics[width=0.17\textwidth]{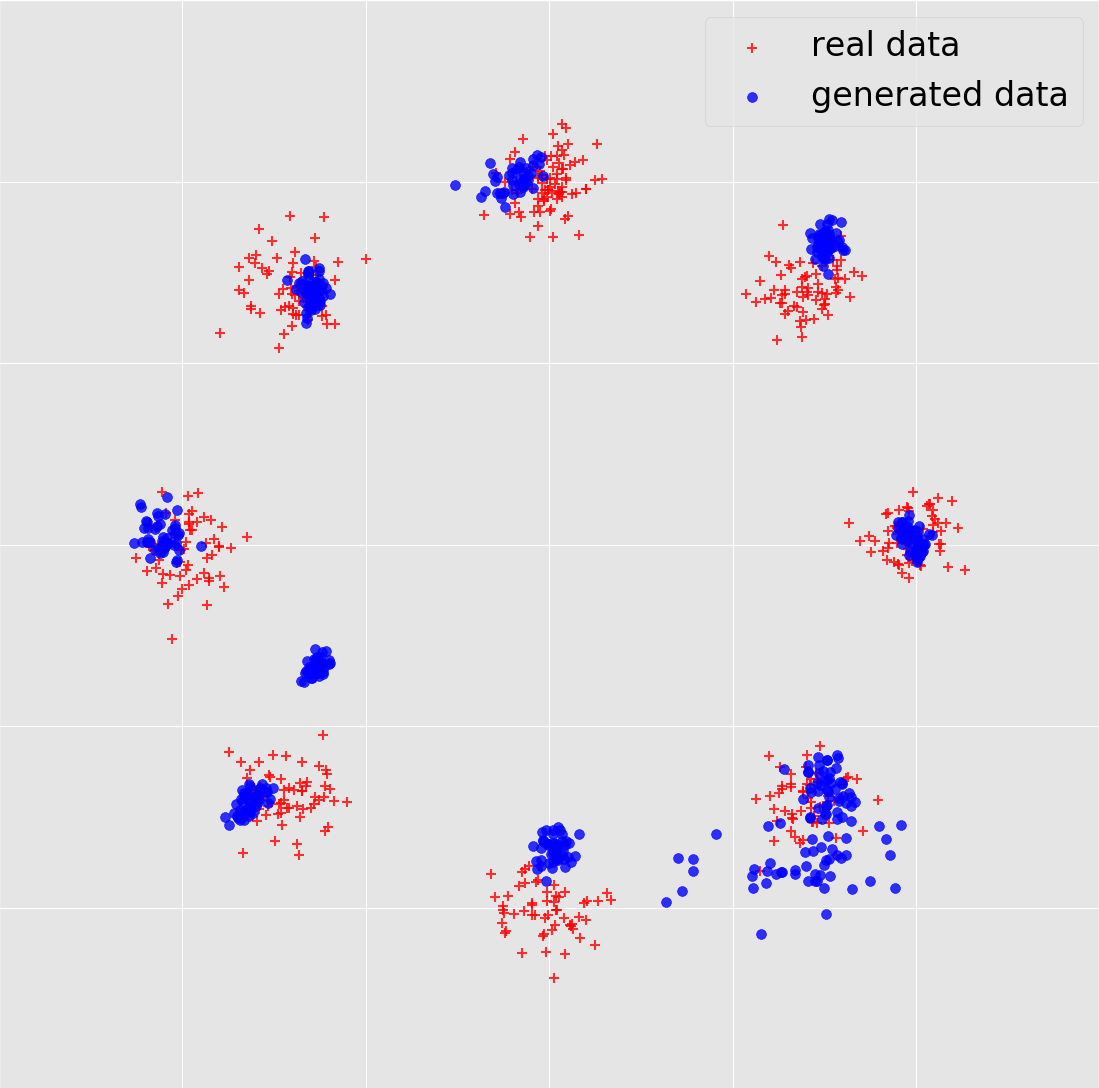} 
\par\end{centering}

\noindent \centering{}}
\par\end{centering}

\noindent \centering{}\caption{Samples generated by $\protect\model$ models trained on synthetic
data with 2, 3, 4 and 10 generators. Generated data are in blue and
data samples from the 8 Gaussians are in red.\label{fig:different_num_gens}}
\end{figure}

\paragraph{The effect of $\beta$ on generated samples.}

To examine the behavior of the diversity coefficient $\beta$, Fig.~\ref{fig:synthetic_beta}
compares samples produced by our $\model$ with 4 generators after
25,000 epochs of training with different values of $\beta$. Without
the JSD force ($\beta=0$), generated samples cluster around one mode.
When $\beta=0.25$, generated data clusters near 4 different modes.
When $\beta=0.75$ or $1.0$, the JSD force is too strong and causes
the generators to collapse, generating 4 increasingly tight clusters.
When $\beta=0.5$, generators successfully cover all of the 8 modes.

\begin{figure}[h]
\noindent \begin{centering}
\subfloat[$\beta=0$]{\noindent \begin{centering}
\includegraphics[width=0.17\textwidth]{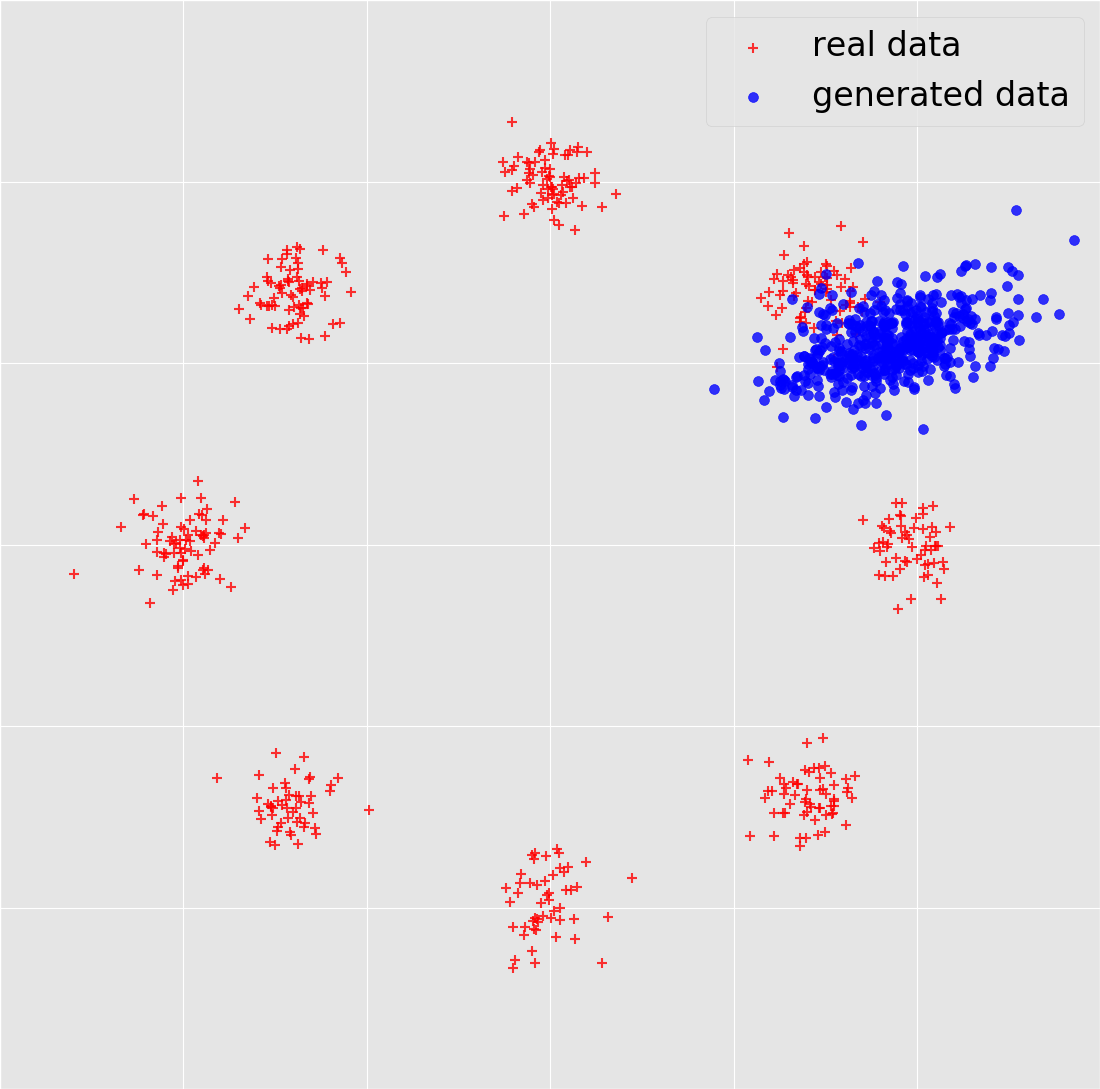} 
\par\end{centering}

\noindent \centering{}}\hfill{} \subfloat[$\beta=0.25$]{\noindent \begin{centering}
\includegraphics[width=0.17\textwidth]{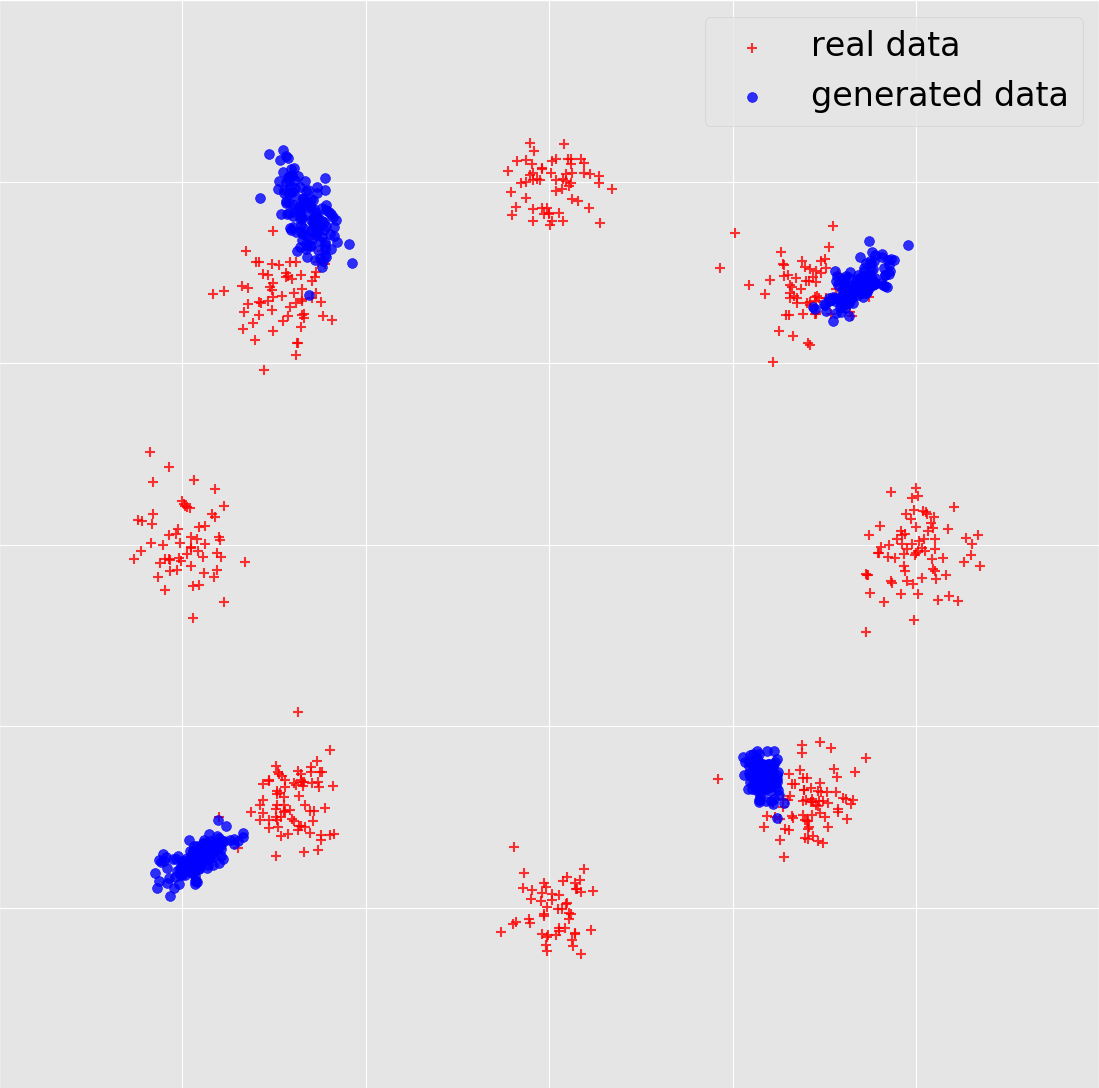}
\par\end{centering}

\noindent \centering{}}\hfill{} \subfloat[$\beta=0.5$]{\noindent \begin{centering}
\includegraphics[width=0.17\textwidth]{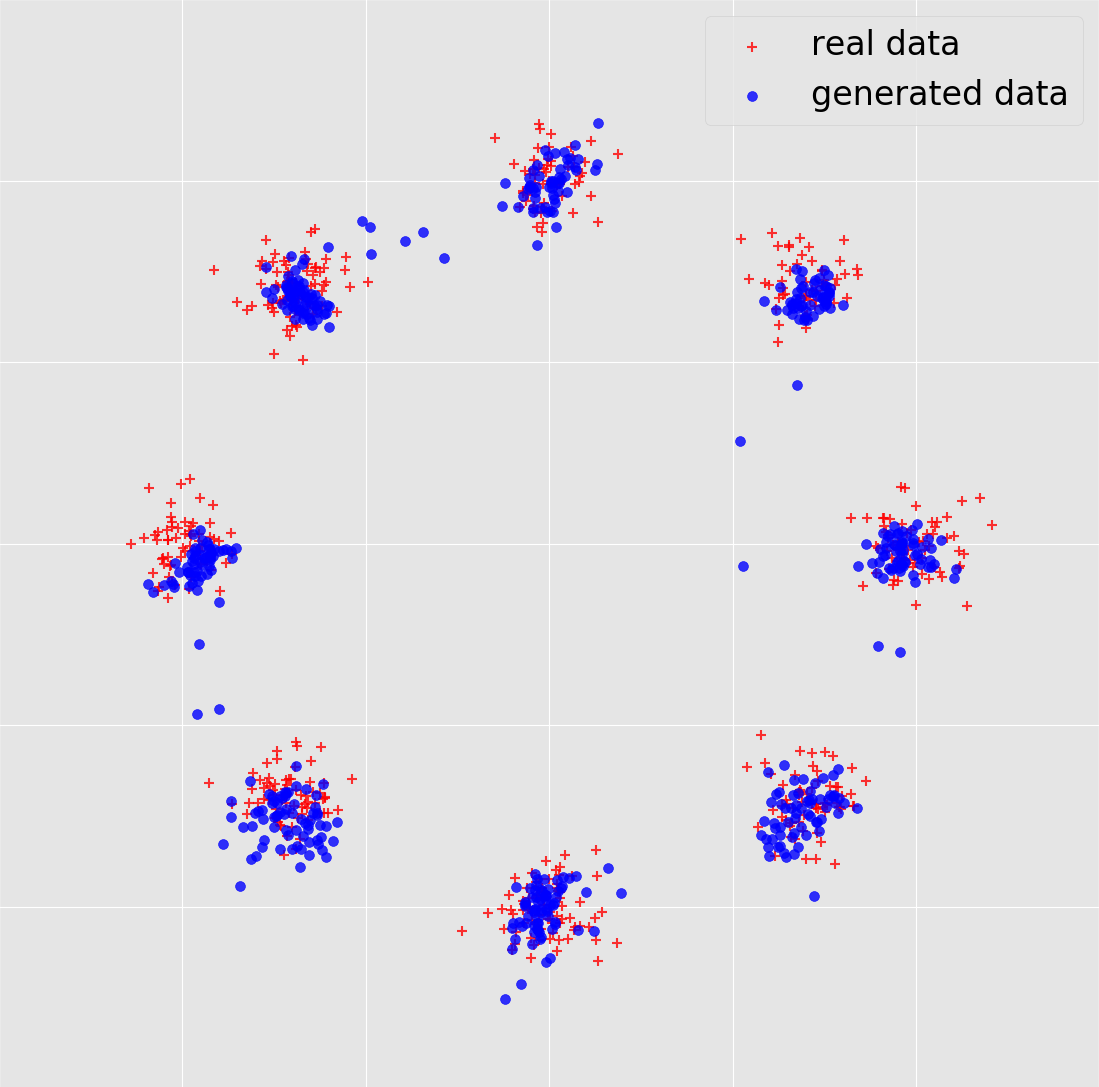}
\par\end{centering}

\noindent \centering{}}\hfill{} \subfloat[$\beta=0.75$]{\noindent \begin{centering}
\includegraphics[width=0.17\textwidth]{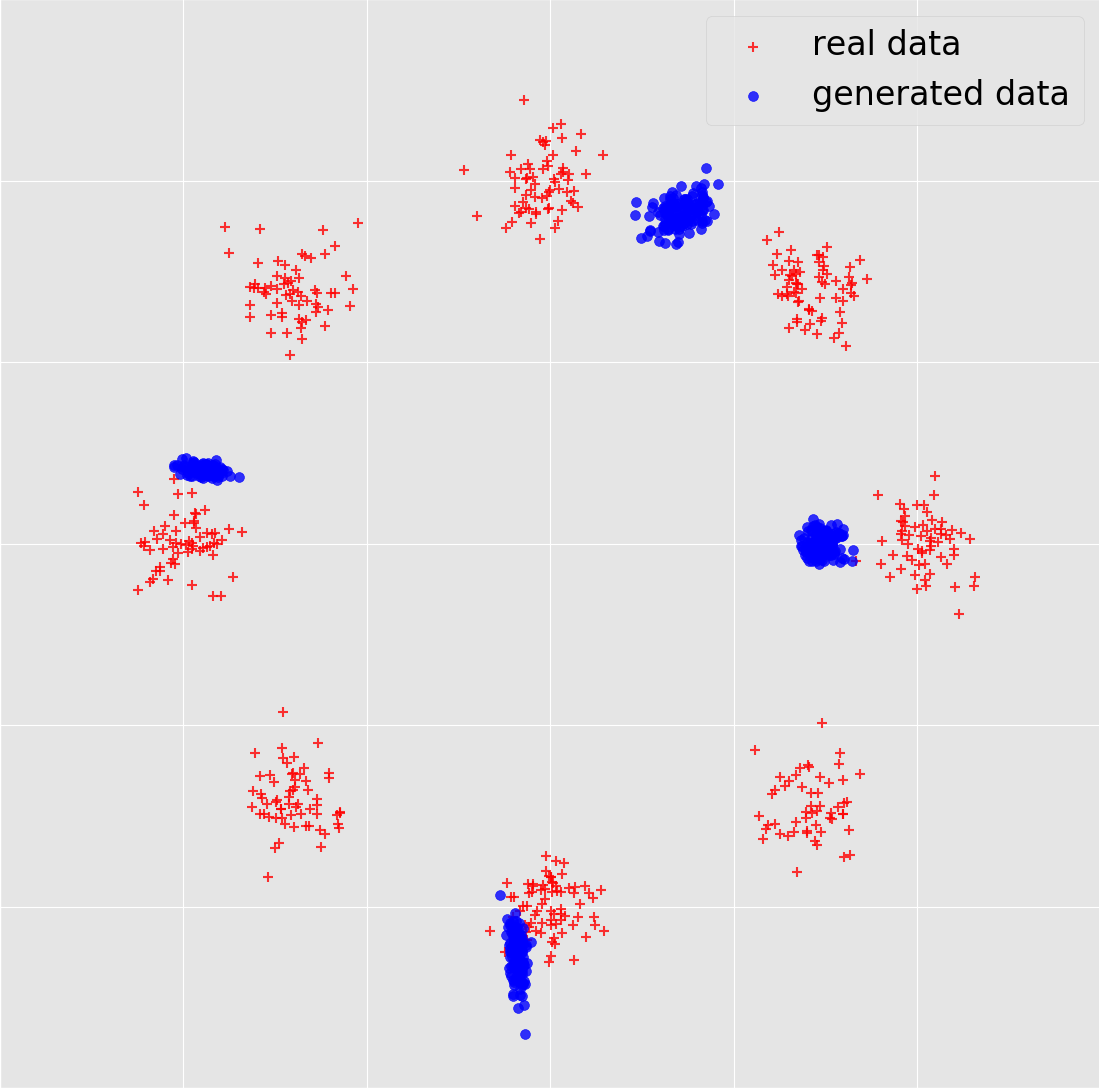}
\par\end{centering}

\noindent \centering{}}\hfill{}\subfloat[$\beta=1.0$]{\noindent \begin{centering}
\includegraphics[width=0.17\textwidth]{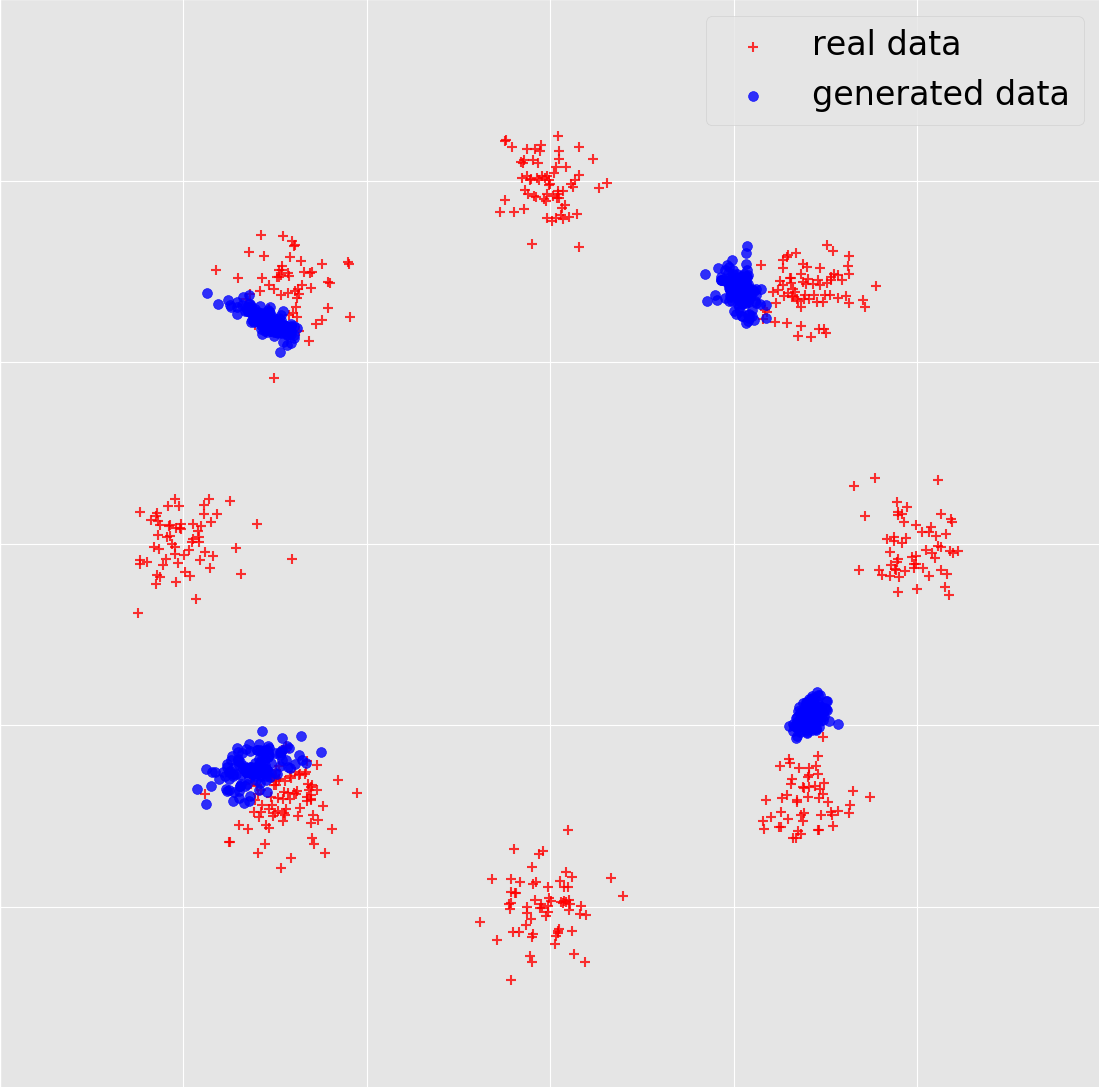}
\par\end{centering}

\noindent \centering{}}
\par\end{centering}

\noindent \centering{}\caption{Samples generated by $\protect\model$ models trained on synthetic
data with different values of diversity coefficient $\beta$. Generated
data are in blue and data samples from the 8 Gaussians are in red.\label{fig:synthetic_beta}}
\end{figure}

\subsection{Real-world datasets\label{sec:Real-world-datasets_appendix}}

\paragraph{Fixing batch normalization center.}

During training we observe that the percentage of active neurons,
which we define as ReLU units with positive activation for at least
10\% of samples in the minibatch, chronically declined. Fig.~\ref{fig:=000025-of-active}
shows the percentage of active neurons in generators trained on CIFAR-10
declined consistently to 55\% in layer 2 and 60\% in layer 3. Therefore,
the quality of generated images, after reaching the peak level, started
declining. One possible cause is that the batch normalization center
(offset) is gradually shifted to the negative range as shown in the
histogram in Fig.~\ref{fig:Histogram-of-batch}. We also observe
the same problem in DCGAN. Our ad-hoc solution for this problem, i.e.,
we fix the offset at zero for all layers in the generator networks.
The rationale is that for each feature map, the ReLU gates will open
for about 50\% highest inputs in a minibatch across all locations
and generators, and close for the rest. Therefore, batch normalization
can keep ReLU units alive even when most of their inputs are otherwise
negative, and introduces a form of competition that encourages generators
to ``specialize'' in different features. This measure significantly
improves performance but does not totally solve the dying ReLUs problem.
We find that late in the training, the input to generators' ReLU units
became more and more right-skewed, causing the ReLU gates to open
less and less often.

\begin{figure}[h]
\noindent \begin{centering}
\subfloat[\% of active neurons in layer 2 and 3.\label{fig:=000025-of-active}]{\includegraphics[width=0.3\textwidth]{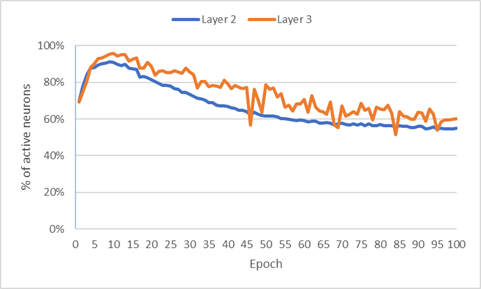}}\qquad{}\subfloat[Histogram of batch normalization centers in layer 2 (left) and 3 (right).\label{fig:Histogram-of-batch}]{\includegraphics[width=0.6\textwidth]{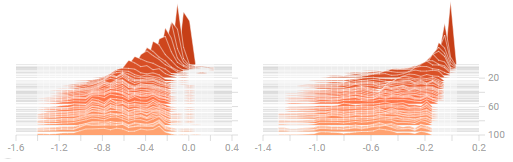}}
\par\end{centering}

\noindent \centering{}\caption{Observation of activate neuron rates and batch normalization centers
in $\protect\model$'s generators trained on CIFAR-10.\label{fig:active_neurons}}
\end{figure}

\paragraph{Experiment settings.}

For the experiments on three large-scale natural scene datasets (CIFAR-10,
STL-10, ImageNet), we closely followed the network architecture and
training procedure of DCGAN. The specifications of our models trained
on CIFAR-10, STL-10 48$\times$48, STL-10 96$\times$96 and ImageNet
datasets are described in Tabs.~(\ref{tab_cifar10_exp}, \ref{tab_stl48_exp},
\ref{tab_stl96_exp}, \ref{tab_imagenet32_exp}), respectively. ``BN''
is short for batch normalization and ``BN center'' is short for
whether to learn batch normalization's center or set it at zero. ``Shared''
is short for parameter sharing among generators or between the classifier
and the discriminator. Feature maps of 10/1 in the last layer for
$C$ and $D$ means that two separate fully connected layers are applied
to the penultimate layer, one for $C$ that outputs $10$ logits and
another for $D$ that outputs $1$ logit. Finally, Figs.~(\ref{fig:cifar-full},
\ref{fig:stl48-full}, \ref{fig:imagenet32-full}, \ref{fig:stl96-good-full},
\ref{fig:stl96-bad-full}) respectively are the enlarged version of
Figs.~(\ref{fig:exp_cifar10_samples}, \ref{fig:exp_stl10_48x48_samples},
\ref{fig:exp_imagenet_samples}, \ref{fig:STL96_good}, \ref{fig:STL96_bad})
in the main manuscript.

\begin{table}[h]
\noindent \begin{centering}
\caption{Network architecture and hyperparameters for the CIFAR-10 dataset.\label{tab_cifar10_exp}}

\par\end{centering}

\noindent \centering{}\resizebox{0.98\textwidth}{!}{
\begin{tabular}{rccrcclc}
\hline 
Operation & Kernel & Strides & Feature maps & BN? & BN center? & Nonlinearity & Shared?\tabularnewline
\hline 
$G\left(\mathbf{z}\right):\mbox{\ensuremath{\mathbf{z}}}\sim\textrm{Uniform}\left[-1,1\right]$ &  &  & 100 &  &  &  & \tabularnewline
Fully connected &  &  & 4$\times$4$\times$512 & $\surd$ & $\times$ & ReLU & $\times$\tabularnewline
Transposed convolution & 5$\times$5 & 2$\times$2 & 256 & $\surd$ & $\times$ & ReLU & $\surd$\tabularnewline
Transposed convolution & 5$\times$5 & 2$\times$2 & 128 & $\surd$ & $\times$ & ReLU & $\surd$\tabularnewline
Transposed convolution & 5$\times$5 & 2$\times$2 & 3 & $\times$ & $\times$ & Tanh & $\surd$\tabularnewline
\hline 
$C\left(\mathbf{x}\right),D\left(\mathbf{x}\right)$ &  &  & 32$\times$32$\times$3 &  &  &  & \tabularnewline
Convolution & 5$\times$5 & 2$\times$2 & 128 & $\surd$ & $\surd$ & Leaky ReLU & $\surd$\tabularnewline
Convolution & 5$\times$5 & 2$\times$2 & 256 & $\surd$ & $\surd$ & Leaky ReLU & $\surd$\tabularnewline
Convolution & 5$\times$5 & 2$\times$2 & 512 & $\surd$ & $\surd$ & Leaky ReLU & $\surd$\tabularnewline
Fully connected &  &  & 10/1 & $\times$ & $\times$ & Softmax/Sigmoid & $\times$\tabularnewline
\hline 
Number of generators & \multicolumn{7}{l}{10}\tabularnewline
Batch size for real data & \multicolumn{7}{l}{64}\tabularnewline
Batch size for each generator & \multicolumn{7}{l}{12}\tabularnewline
Number of iterations & \multicolumn{7}{l}{250}\tabularnewline
Leaky ReLU slope & \multicolumn{7}{l}{0.2}\tabularnewline
Learning rate & \multicolumn{7}{l}{0.0002}\tabularnewline
Regularization constants & \multicolumn{7}{l}{$\beta=0.01$}\tabularnewline
Optimizer & \multicolumn{7}{l}{Adam$\left(\beta_{1}=0.5,\beta_{2}=0.999\right)$}\tabularnewline
Weight, bias initialization & \multicolumn{7}{l}{$\mathcal{N}\left(\mu=0,\sigma=0.01\right)$, $0$}\tabularnewline
\hline 
\end{tabular}}
\end{table}

\begin{table}[h]
\noindent \begin{centering}
\caption{Network architecture and hyperparameters for the STL-10 48$\times$48
dataset.\label{tab_stl48_exp}}

\par\end{centering}

\noindent \centering{}\resizebox{0.98\textwidth}{!}{
\begin{tabular}{rccrcclc}
\hline 
Operation & Kernel & Strides & Feature maps & BN? & BN center? & Nonlinearity & Shared?\tabularnewline
\hline 
$G\left(\mathbf{z}\right):\mbox{\ensuremath{\mathbf{z}}}\sim\textrm{Uniform}\left[-1,1\right]$ &  &  & 100 &  &  &  & \tabularnewline
Fully connected &  &  & 4$\times$4$\times$1024 & $\surd$ & $\times$ & ReLU & $\times$\tabularnewline
Transposed convolution & 5$\times$5 & 2$\times$2 & 512 & $\surd$ & $\times$ & ReLU & $\surd$\tabularnewline
Transposed convolution & 5$\times$5 & 2$\times$2 & 256 & $\surd$ & $\times$ & ReLU & $\surd$\tabularnewline
Transposed convolution & 5$\times$5 & 2$\times$2 & 128 & $\surd$ & $\times$ & ReLU & $\surd$\tabularnewline
Transposed convolution & 5$\times$5 & 2$\times$2 & 3 & $\times$ & $\times$ & Tanh & $\surd$\tabularnewline
\hline 
$C\left(\mathbf{x}\right),D\left(\mathbf{x}\right)$ &  &  & 48$\times$48$\times$3 &  &  &  & \tabularnewline
Convolution & 5$\times$5 & 2$\times$2 & 128 & $\surd$ & $\surd$ & Leaky ReLU & $\surd$\tabularnewline
Convolution & 5$\times$5 & 2$\times$2 & 256 & $\surd$ & $\surd$ & Leaky ReLU & $\surd$\tabularnewline
Convolution & 5$\times$5 & 2$\times$2 & 512 & $\surd$ & $\surd$ & Leaky ReLU & $\surd$\tabularnewline
Convolution & 5$\times$5 & 2$\times$2 & 1024 & $\surd$ & $\surd$ & Leaky ReLU & $\surd$\tabularnewline
Fully connected &  &  & 10/1 & $\times$ & $\times$ & Softmax/Sigmoid & $\times$\tabularnewline
\hline 
Number of generators & \multicolumn{7}{l}{10}\tabularnewline
Batch size for real data & \multicolumn{7}{l}{64}\tabularnewline
Batch size for each generator & \multicolumn{7}{l}{12}\tabularnewline
Number of iterations & \multicolumn{7}{l}{250}\tabularnewline
Leaky ReLU slope & \multicolumn{7}{l}{0.2}\tabularnewline
Learning rate & \multicolumn{7}{l}{0.0002}\tabularnewline
Regularization constants & \multicolumn{7}{l}{$\beta=1.0$}\tabularnewline
Optimizer & \multicolumn{7}{l}{Adam$\left(\beta_{1}=0.5,\beta_{2}=0.999\right)$}\tabularnewline
Weight, bias initialization & \multicolumn{7}{l}{$\mathcal{N}\left(\mu=0,\sigma=0.01\right)$, $0$}\tabularnewline
\hline 
\end{tabular}}
\end{table}

\begin{table}[h]
\noindent \begin{centering}
\caption{Network architecture and hyperparameters for the STL96$\times$96
dataset.\label{tab_stl96_exp}}

\par\end{centering}

\noindent \centering{}\resizebox{0.98\textwidth}{!}{
\begin{tabular}{rccrcclc}
\hline 
Operation & Kernel & Strides & Feature maps & BN? & BN center? & Nonlinearity & Shared?\tabularnewline
\hline 
$G\left(\mathbf{z}\right):\mbox{\ensuremath{\mathbf{z}}}\sim\textrm{Uniform}\left[-1,1\right]$ &  &  & 100 &  &  &  & \tabularnewline
Fully connected &  &  & 4$\times$4$\times$2046 & $\surd$ & $\times$ & ReLU & $\times$\tabularnewline
Transposed convolution & 5$\times$5 & 2$\times$2 & 1024 & $\surd$ & $\times$ & ReLU & $\surd$\tabularnewline
Transposed convolution & 5$\times$5 & 2$\times$2 & 512 & $\surd$ & $\times$ & ReLU & $\surd$\tabularnewline
Transposed convolution & 5$\times$5 & 2$\times$2 & 256 & $\surd$ & $\times$ & ReLU & $\surd$\tabularnewline
Transposed convolution & 5$\times$5 & 2$\times$2 & 128 & $\surd$ & $\times$ & ReLU & $\surd$\tabularnewline
Transposed convolution & 5$\times$5 & 2$\times$2 & 3 & $\times$ & $\times$ & Tanh & $\surd$\tabularnewline
\hline 
$C\left(\mathbf{x}\right),D\left(\mathbf{x}\right)$ &  &  & 32$\times$32$\times$3 &  &  &  & \tabularnewline
Convolution & 5$\times$5 & 2$\times$2 & 128 & $\surd$ & $\surd$ & Leaky ReLU & $\surd$\tabularnewline
Convolution & 5$\times$5 & 2$\times$2 & 256 & $\surd$ & $\surd$ & Leaky ReLU & $\surd$\tabularnewline
Convolution & 5$\times$5 & 2$\times$2 & 512 & $\surd$ & $\surd$ & Leaky ReLU & $\surd$\tabularnewline
Convolution & 5$\times$5 & 2$\times$2 & 1024 & $\surd$ & $\surd$ & Leaky ReLU & $\surd$\tabularnewline
Convolution & 5$\times$5 & 2$\times$2 & 2048 & $\surd$ & $\surd$ & Leaky ReLU & $\surd$\tabularnewline
Fully connected &  &  & 10/1 & $\times$ & $\times$ & Softmax/Sigmoid & $\times$\tabularnewline
\hline 
Number of generators & \multicolumn{7}{l}{10}\tabularnewline
Batch size for real data & \multicolumn{7}{l}{64}\tabularnewline
Batch size for each generator & \multicolumn{7}{l}{12}\tabularnewline
Number of iterations & \multicolumn{7}{l}{250}\tabularnewline
Leaky ReLU slope & \multicolumn{7}{l}{0.2}\tabularnewline
Learning rate & \multicolumn{7}{l}{0.0002}\tabularnewline
Regularization constants & \multicolumn{7}{l}{$\beta=1.0$}\tabularnewline
Optimizer & \multicolumn{7}{l}{Adam$\left(\beta_{1}=0.5,\beta_{2}=0.999\right)$}\tabularnewline
Weight, bias initialization & \multicolumn{7}{l}{$\mathcal{N}\left(\mu=0,\sigma=0.01\right)$, $0$}\tabularnewline
\hline 
\end{tabular}}
\end{table}

\begin{table}[h]
\noindent \begin{centering}
\caption{Network architecture and hyperparameters for the ImageNet dataset.\label{tab_imagenet32_exp}}

\par\end{centering}

\noindent \centering{}\resizebox{0.98\textwidth}{!}{
\begin{tabular}{rccrcclc}
\hline 
Operation & Kernel & Strides & Feature maps & BN? & BN center? & Nonlinearity & Shared?\tabularnewline
\hline 
$G\left(\mathbf{z}\right):\mbox{\ensuremath{\mathbf{z}}}\sim\textrm{Uniform}\left[-1,1\right]$ &  &  & 100 &  &  &  & \tabularnewline
Fully connected &  &  & 4$\times$4$\times$512 & $\surd$ & $\times$ & ReLU & $\times$\tabularnewline
Transposed convolution & 5$\times$5 & 2$\times$2 & 256 & $\surd$ & $\times$ & ReLU & $\surd$\tabularnewline
Transposed convolution & 5$\times$5 & 2$\times$2 & 128 & $\surd$ & $\times$ & ReLU & $\surd$\tabularnewline
Transposed convolution & 5$\times$5 & 2$\times$2 & 3 & $\times$ & $\times$ & Tanh & $\surd$\tabularnewline
\hline 
$C\left(\mathbf{x}\right),D\left(\mathbf{x}\right)$ &  &  & 32$\times$32$\times$3 &  &  &  & \tabularnewline
Convolution & 5$\times$5 & 2$\times$2 & 128 & $\surd$ & $\surd$ & Leaky ReLU & $\surd$\tabularnewline
Convolution & 5$\times$5 & 2$\times$2 & 256 & $\surd$ & $\surd$ & Leaky ReLU & $\surd$\tabularnewline
Convolution & 5$\times$5 & 2$\times$2 & 512 & $\surd$ & $\surd$ & Leaky ReLU & $\surd$\tabularnewline
Fully connected &  &  & 10/1 & $\times$ & $\times$ & Softmax/Sigmoid & $\times$\tabularnewline
\hline 
Number of generators & \multicolumn{7}{l}{10}\tabularnewline
Batch size for real data & \multicolumn{7}{l}{64}\tabularnewline
Batch size for each generator & \multicolumn{7}{l}{12}\tabularnewline
Number of iterations & \multicolumn{7}{l}{50}\tabularnewline
Leaky ReLU slope & \multicolumn{7}{l}{0.2}\tabularnewline
Learning rate & \multicolumn{7}{l}{0.0002}\tabularnewline
Regularization constants & \multicolumn{7}{l}{$\beta=0.1$}\tabularnewline
Optimizer & \multicolumn{7}{l}{Adam$\left(\beta_{1}=0.5,\beta_{2}=0.999\right)$}\tabularnewline
Weight, bias initialization & \multicolumn{7}{l}{$\mathcal{N}\left(\mu=0,\sigma=0.01\right)$, $0$}\tabularnewline
\hline 
\end{tabular}}
\end{table}

\begin{figure}[h]
\noindent \begin{centering}
\includegraphics[width=0.9\textwidth]{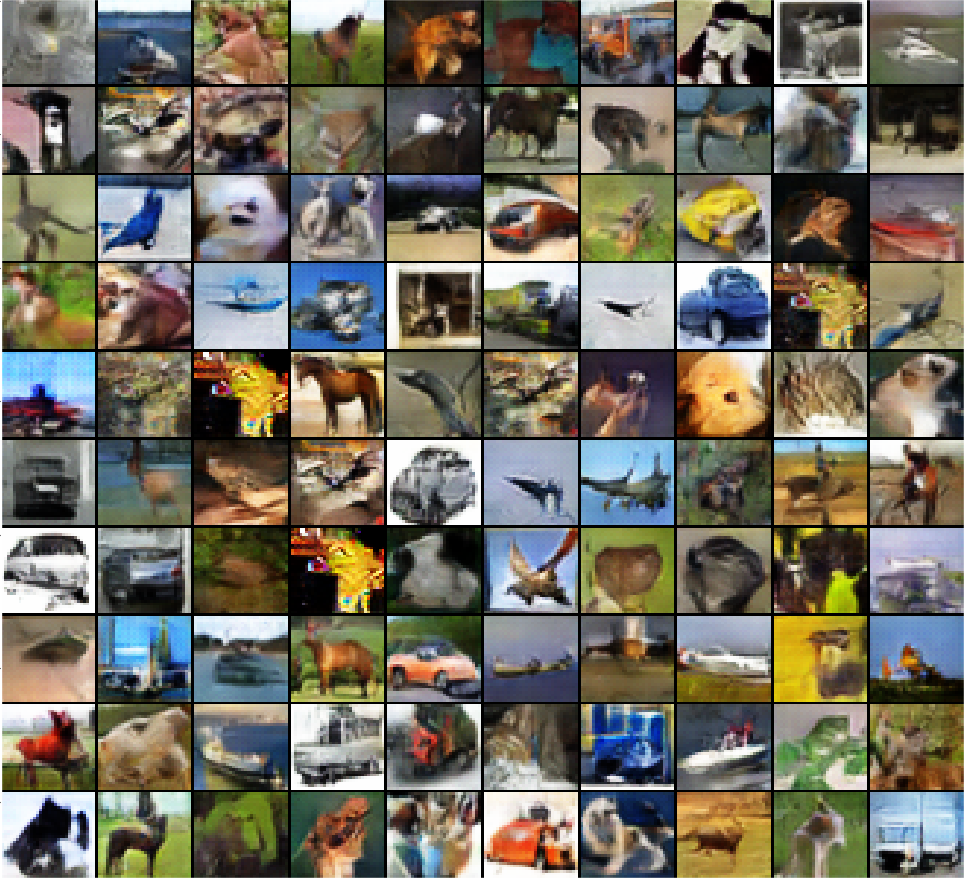}
\par\end{centering}

\noindent \centering{}\caption{Images generated by $\protect\model$ trained on the CIFAR-10 dataset.\label{fig:cifar-full}}
\end{figure}

\begin{figure}[h]
\noindent \begin{centering}
\includegraphics[width=0.9\textwidth]{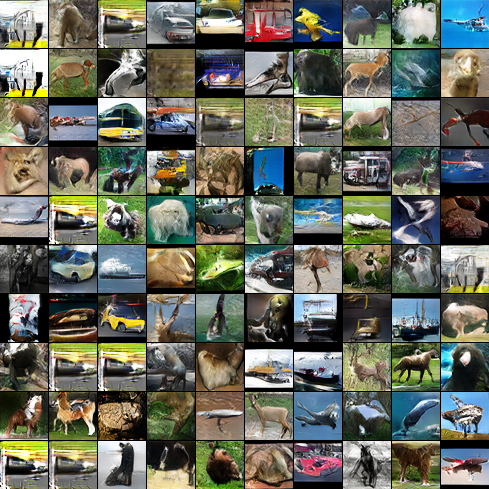}
\par\end{centering}

\noindent \centering{}\caption{Images generated by $\protect\model$ trained on the rescaled 48$\times$48
STL-10 dataset.\label{fig:stl48-full}}
\end{figure}

\begin{figure}[h]
\noindent \begin{centering}
\includegraphics[width=0.9\textwidth]{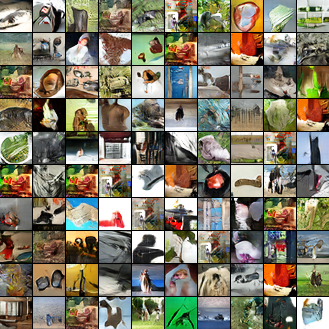}
\par\end{centering}

\noindent \centering{}\caption{Images generated by $\protect\model$ trained on the rescaled 32$\times$32
ImageNet dataset.\label{fig:imagenet32-full}}
\end{figure}

\begin{figure}[h]
\noindent \begin{centering}
\includegraphics[width=0.9\textwidth]{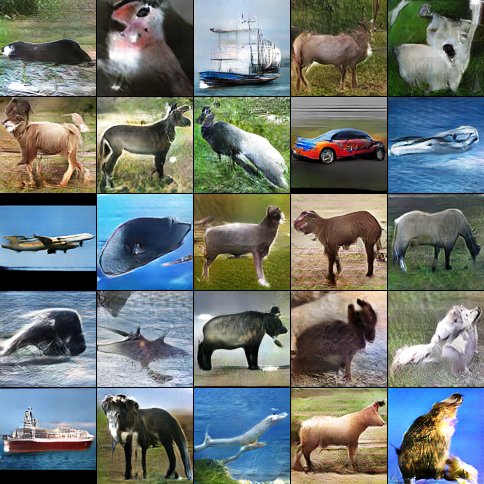}
\par\end{centering}

\noindent \centering{}\caption{Cherry-picked samples generated by $\protect\model$ trained on the
96$\times$96 STL-10 dataset.\label{fig:stl96-good-full}}
\end{figure}

\begin{figure}[h]
\noindent \begin{centering}
\includegraphics[width=0.9\textwidth]{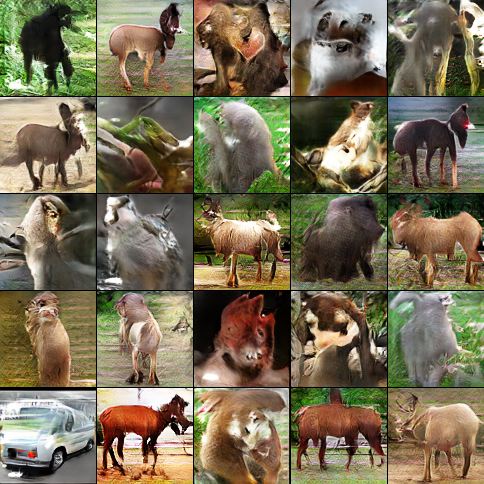}
\par\end{centering}

\noindent \centering{}\caption{Incomplete, unrealistic samples generated by $\protect\model$ trained
on the 96$\times$96 STL-10 dataset.\label{fig:stl96-bad-full}}
\end{figure}

\end{document}